\documentclass[runningheads]{llncs}
\usepackage{booktabs}   
\usepackage{subcaption} 
\usepackage{listings}
\usepackage{relsize}
\usepackage{enumitem}
\usepackage{float}
\usepackage{nicefrac}

\usepackage[algosection,ruled,vlined,linesnumbered]{algorithm2e}
  \SetKwInput{KwIn}{Inputs}
  \SetKwFunction{Fselfrepair}{SC}
  \SetKwFunction{Frepair}{SC-Layer}
  \SetKwFunction{Fsolve}{FindSatConstraint}
  \SetKwFunction{Freorder}{Reorder}
  \SetKwFunction{Ftopsort}{TopologicalSort}
  \SetKwFunction{Fcheckcycle}{ContainsCycle}
  \SetKwFunction{Ffilter}{Filter}
  \SetKwFunction{Ftodnf}{ToDNF}
  \SetKwFunction{Fprioritize}{Prioritize}
  \SetKwFunction{Fordergraph}{OrderGraph}
  \SetKwFunction{Fcontainscycle}{ContainsCycle}
  \SetKwFunction{Fissat}{IsSat}
  \SetKwFunction{Freorder}{Correct}
  \SetKwFunction{Fdescsort}{SortDescending}
  \SetKwFunction{Finvperm}{InvertPermutation}
  \SetKwProg{Fn}{}{:}{}
  \DontPrintSemicolon
  \newlength{\commentWidth}
  \newcommand{\atcp}[1]{\tcp*[r]{\makebox[\commentWidth]{#1\hfill}}}

\usepackage{tikz}
  \usetikzlibrary{backgrounds,calc,fit,positioning}
\usepackage{xspace}
\usepackage{braket}
\usepackage{nameref}

\usepackage{xparse}
\DeclareDocumentCommand{\newmathcommand}{mO{0}m}{%
    \expandafter\let\csname old\string#1\endcsname=#1
    \expandafter\newcommand\csname new\string#1\endcsname[#2]{#3}
    \DeclareRobustCommand#1{%
        \ifmmode
        \expandafter\let\expandafter\next\csname new\string#1\endcsname
        \else
        \expandafter\let\expandafter\next\csname old\string#1\endcsname
        \fi
        \next
    }%
}
\usepackage{amsmath}
\usepackage{amssymb}
\usepackage{hyperref}

\spnewtheorem{Property}[theorem]{Property}{\bfseries}{\itshape}
\spnewtheorem{invariant}[theorem]{Invariant}{\bfseries}{\itshape}

\newcommand{\properties}{\ensuremath{\Phi}\xspace}

\newcommand{\pre}{\ensuremath{P}\xspace}
\newcommand{\post}{\ensuremath{Q}\xspace}
\newcommand{\pred}{\ensuremath{q}\xspace}

\newcommand{\arnet}{SC-Net\xspace}
\newcommand{\arnets}{\arnet{}s\xspace}
\newmathcommand{\arnet}{f^\properties}
\def\arlayer{SC-Layer\xspace}

\def\perm{{\tt perm}\xspace}
\def\idx{{\tt idx}\xspace}

\newcommand{\rootg}{\text{Roots}\xspace}
\newcommand{\eat}[1]{}

\newcommand{\trace}{\text{Trace}}

\DeclareMathOperator*{\argmax}{argmax}

\definecolor{blue2}{rgb}{0.0, 0.5, 1.0}

\let\ls\lstinline
\begin{document}
\title{Self-Correcting Neural Networks For Safe Classification}
%
%
\author{
Klas~Leino\inst{1}
\and
Aymeric~Fromherz\inst{2}
\and
Ravi~Mangal\inst{1}
\and
Matt~Fredrikson\inst{1}
\and
Bryan~Parno\inst{1}
\and
Corina~P\u{a}s\u{a}reanu\inst{1}
}
\authorrunning{Leino, Fromherz, Mangal, Fredrikson, Parno, and P\u{a}s\u{a}reanu}
%
\institute{Carnegie Mellon University \\
\and INRIA Paris \\
\email{kleino@cs.cmu.edu}\\
\email{aymeric.fromherz@inria.fr}\\
\email{rmangal@andrew.cmu.edu}\\
\email{mfredrik@cmu.edu}\\
\email{parno@cmu.edu}\\
\email{pcorina@cmu.edu}}
\maketitle
\begin{abstract}

Classifiers learnt from data are increasingly being used as components in systems where safety is a critical concern.
In this work, we present a formal notion of safety for classifiers via constraints called \emph{safe-ordering constraints}. These constraints 
relate requirements on the order of the classes output by a classifier to conditions on its input, and are expressive enough to encode various interesting examples of classifier safety specifications from the literature.
For classifiers implemented using neural networks, we also present a run-time mechanism for the enforcement of safe-ordering constraints.
Our approach is based on a \emph{self-correcting layer}, which provably yields safe outputs regardless of the characteristics of the classifier input.
We compose this layer with an existing neural network classifier to construct a \emph{self-correcting network} (SC-Net), and show that in addition to providing safe outputs, the SC-Net is guaranteed to preserve the classification accuracy of the original network whenever possible. 
Our approach is independent of the size and architecture of the neural network used for classification, depending only on the specified property and the dimension of the network's output; thus it is scalable to large state-of-the-art networks.
We show that our approach can be 
optimized for a GPU,\footnote{Code available at \url{github.com/cmu-transparency/self-correcting-networks}.} introducing run-time overhead of less than 1ms on current hardware---even on large, widely-used networks containing hundreds of thousands of neurons and millions of parameters.

\end{abstract}



\section{Introduction}
\label{sec:intro}
Classifiers in the form of neural networks are being deployed as components in many safety- and
security-critical systems, such as autonomous vehicles, banking systems, and medical diagnostics. 
A well-studied example is the ACAS Xu networks~\cite{julian2019acas}, 
which provide guidance to an airborne collision avoidance system for commercial aircraft.
Unfortunately, standard network training approaches will
typically produce models that are accurate but unsafe~\cite{mirman2018differentiable,lin20}.
The ACAS Xu networks, in particular,
have been shown~\cite{katz17} to violate safety properties formulated by the developers~\cite{julian2019acas}.

\paragraph{What are safety properties for classifiers?} Classifiers implemented as neural networks are programs of type 
$\mathbb{R}^n
\to \mathbb{R}^m$, where typically the index of the maximum element of the output $m$-tuple represents the predicted class.
Such classifiers also give an order on the classes, from most likely to least, represented by the order on indices induced by sorting the elements (also referred to as \emph{logits}) of the tuple, and in a variety of domains, systems with classifier components may use this ordering, in addition to the top predicted class, for downstream decision-making. 

The ACAS Xu classifiers are an example of a domain where ordering matters. They map sensor readings about the physical state of the aircraft to horizontal maneuver advisories.
The sensor readings are imperfect, and the system only has access to a distribution function (or, alternatively, a set of samples) that assigns probability $b(s)$ to being in state $s$. 
To issue a maneuver guidance in real time, at each time-step, the system finds the maneuver that maximizes $\sum_s Q(s)_a b(s)$ 
where $Q(s)_a$ is the value assigned by the neural classifier to maneuver $a$ in state $s$. As a consequence, the order of the classes, in addition to the top class, are relevant when defining safety properties of ACAS Xu networks. 

Another example domain is image classification, where popular datasets, such as CIFAR-100 and ImageNet, have classes with hierarchical structure (e.g., CIFAR-100 has 100 classes with 20 superclasses). Consider a client of an image classifier that averages the logit values over a number of samples for classes that appear in top-k positions and chooses the class with the highest average logit value, due to imperfect sensor information. A reasonable safety property is to require that the chosen class shares its superclass with at least one of the top-1 predictions. This in turn requires reasoning over the order of the classes, and not just the top class.

More generally, the ordering of the logits conveys information about the neural classifier’s `belief' in what the true class is. Under this interpretation, it is natural to express safety constraints on the class order.
On the other hand, the exact logit values may be less meaningful, given the approximate nature of neural networks and the fact that logit values are not typically calibrated to any particular value.

Motivated by these observations, we define safety property specifications for classifiers via constraints 
that we refer to as \emph{safe-ordering constraints}. 
We argue that these constraints are general enough to encompass the meaningful safety specifications
defined for the ACAS Xu networks~\cite{katz17},
as well as those used in other safety verification and repair efforts~\cite{lin20,sotoudeh2021provable}.
Formally, safe-ordering constraints can specify non-relational safety properties~\cite{clarkson08} of the form $P \implies Q$, where $P$ is a precondition, 
expressed as a decidable formula over the classifier's input, and $Q$ is a postcondition, expressed as a statement over its output in the theory of totally ordered sets.

We note that many safety specifications provided by experts are \emph{underconstrained}~\cite{katz17};
i.e., they say what the classifier \emph{should not do} but not what it \emph{should} do.
As a specific example, one of the ACAS Xu safety properties 
roughly states that if an oncoming aircraft is directly ahead
and is moving toward our aircraft, then the clear-of-conflict 
advisory should not have the maximal output from the network\footnote{While \cite{julian2019acas} used the convention that the index of the minimal element of ACAS Xu networks is the top predicted advisory, in this paper we will use the more common convention of the maximal value's index.}.
The decision as to what output should have the maximal value
must be determined by learning from the input-output examples in the training data.
Thus, even when safety specifications are provided, one still needs to perform training based on labeled data to build the classifier, whose performance is measured by computing its accuracy on a separate test set.

\paragraph{Enforcing Safe-Ordering Constraints} Standard approaches for learning neural classifiers will
typically produce models that are accurate but unsafe~\cite{mirman2018differentiable,lin20}.
As a result, considerable work has studied 
the safety of neural networks in general~\cite{huang2017safety,gehr18,dvijotham18,mirman2018differentiable,singh19,anderson19,muller2021scaling,ehlers2017formal},
and the ACAS Xu networks in particular~\cite{katz17,sotoudeh2021provable,mirman2018differentiable,lin20,wu2020parallelization}.

Some approaches use 
abstract interpretation~\cite{gehr18,singh19} or SMT solving~\cite{katz17} 
to verify safety properties of networks trained using standard techniques.
Unfortunately, the scalability of these techniques remains a serious challenge for most neural-network applications.
Furthermore, post-training verification does not address the problem of constructing safe networks to begin with.
Retraining the network when verification fails is prohibitively expensive for modern networks~\cite{tan19,brown20,tan21}, with no guarantees that the train-verify-train loop will terminate.
On the other hand, approaches based on statically repairing the network can damage its accuracy (i.e., frequency of the top predicted class matching the `true' class) 
on inputs outside the scope of a given safety specification~\cite{sotoudeh2021provable}.
An alternate approach is to change the learning algorithm such that it provably produces safe networks~\cite{lin20},
but such approaches may not converge during training, thus not being able to provide a safety guarantee for the analyzed networks.

In contrast to these previous works, 
we propose a lightweight, run-time technique for ensuring that neural classifiers are {\em guaranteed} to satisfy their safe-ordering specifications and at the same time maintain the network's accuracy. 
Specifically, we describe a program transformer that, 
given a neural architecture $f_\theta$ (parameterized by $\theta$) and a set of safe-ordering constraints $\Phi$, 
produces a new architecture $f_\theta^\Phi$ that satisfies the
conjunction of $\Phi$ for all parameters $\theta$.
Viewing the neural network as a composition of layers, our transformer appends
a differentiable \emph{self-correcting layer} (\arlayer) to $f_\theta$. 
This layer encodes a dynamic \textit{check-and-correct} mechanism, so that when $f_\theta(x)$ violates $\Phi$, the \arlayer modifies the output to ensure safety.
Differentiability of the mechanism also opens the possibility for the training procedure to take self-correction into account during training so that safer and more accurate models can be built, reducing the need for the run-time correction.

Consider again the  ACAS Xu networks. Ideally, before deploying the system, we would like to certify that the trained neural classifiers meet their safety specifications. 
Since the training algorithms are not guaranteed to produce safe classifiers~\cite{mirman2018differentiable,lin20}, and the train-verify-train loop may not terminate, one is likely to be forced to deploy uncertified classifiers.
A run-time mechanism that flags safety violations can provide some assurance, but for a real-time, unmanned system like ACAS Xu, throwing exceptions during operation and aborting the computation is not acceptable.
Instead, to ensure safe operation without interruptions, we propose to correct the outputs of the classifier whenever necessary.

Our approach is similar in spirit to those that dynamically correct errors in long-running programs caused by traditional software issues like division-by-zero, null dereference, and others~\cite{long14,rinard04,kling12,berger06,qin05,perkins09}, as well as dynamic check-and-correct mechanisms employed by controllers, referred to as shields \cite{bloem2015shield,alshiekh2018safe,zhu19inductive}.
A check-and-correct mechanism may be impractical for arbitrary classifier safety specifications, as they may require solving arbitrarily complex constraint-satisfaction problems.
We show that this is not the case for safe-ordering constraints, and that
the solver needed for these constraints can be efficiently embedded in the correction layer.

We note that when correcting the neural network output to enforce safety, we still need to preserve 
its accuracy. 
To address the issue, we define a property, \textit{transparency}, which ensures that the correction mechanism has no negative impact on the network's accuracy.
Transparency requires that the predicted top class of the original network $f_\theta$ be retained whenever it is consistent with at least one ordering allowed by $\Phi$. However, if $\Phi$ is inconsistent with the ``correct'' class specified by the data, then it is impossible for the network to be safe without harming accuracy, and the correction prioritizes safety.  
We prove that our \arlayer guarantees transparency. More generally, our correction mechanism tries to retain as much of the original class order as possible.

Finally, while the \arlayer achieves safety without negatively impacting accuracy, it necessarily adds computational overhead each time the network is executed.
We design the \arlayer, including the embedded constraint solver, to be both vectorized and differentiable, allowing the efficient implementation of our approach within popular neural network frameworks.
We also present experiments that evaluate how the overhead is impacted by several key factors.
We show that the cost of the \arlayer depends solely on $\Phi$ and the length $m$ of the output vector, and thus, is \emph{independent} of the size or complexity of the underlying neural network. 
On three widely-used benchmark datasets (ACAS Xu~\cite{katz17}, Collision Detection~\cite{ehlers2017formal}, and CIFAR-100~\cite{krizhevsky09}), we show that this overhead is small in real terms (0.26-0.82 ms), and does not pose an impediment to practical adoption.
In fact, because the overhead is independent of network size, its impact is less noticeable on larger networks, where the cost of evaluating the original classifier may come to dominate that of the correction.
To further characterize the role of $\Phi$ and $m$, we use synthetic data and random safe-ordering constraints to isolate the effects that the postcondition complexity and number of classes have on network run time.
While these structural traits of the specified safety constraint can impact run time---the satisfiability of general ordering constraints is NP-complete~\cite{guttman06variations}---our results suggest it will be rare in practice.

Hence, the main contributions of our work are as follows:
\begin{itemize}
\item 
We define \emph{safe-ordering constraints},
as a generic way of writing safety specifications for neural network classifiers.
\item 
We present a method for transforming feed-forward neural network architectures into safe-by-construction versions that are guaranteed to \emph{(i)} satisfy a given set of safe-ordering constraints, and \emph{(ii)} preserve or improve the empirical accuracy of the original model.
\item 
We show that the \arlayer can be designed to be both fully-vectorized and differentiable, which enables hardware acceleration to reduce run-time overhead,
and facilitates its use during training.
\item 
We empirically demonstrate that the overhead introduced by the \arlayer is small enough for its deployment in practical settings.
\end{itemize}


\section{Problem Setting}
\label{sec:properties}

In this section, we formalize the concepts of \emph{safe-ordering constraints} and \emph{self-correction}.
We begin by presenting background on neural networks and an illustrative application of safe-ordering constraints.
We then formally define the problem we aim to solve, and introduce a set of desired properties for our self-correcting transformer.

\subsection{Background}
\label{sec:properties:back}

\paragraph{Neural Networks}
A neural network, $f_\theta : \mathbb{R}^n \to \mathbb{R}^m$, is a total function defined by an \emph{architecture}, or composition of linear and non-linear transformations,
and a set of \emph{weights}, $\theta$, parameterizing its linear transformations.
As neither the details of a network's architecture nor the particular valuation of its weights are relevant to much of this paper, we will by default omit the subscript $\theta$, and treat $f$ as a black-box function.
Neural networks are used as classifiers by extracting \emph{class predictions}
from the output $f(x) : \mathbb{R}^m$, also called the \emph{logits} of a network.
Given a neural network $f$, we use the upper-case $F$ to refer to the corresponding neural classifier that returns the top class: $F = \lambda x.\argmax_i\{f_i(x)\}$.
For our purposes, we will assume that $\argmax$ returns a single index, $i^* \in [m]$\footnote{$[m] := \{0,\ldots,m-1\}$}; ties may be broken arbitrarily.

\paragraph{ACAS Xu: An Illustrative Example}
We use ACAS Xu~\cite{julian2019acas} as a running example to illustrate key aspects of the problem that our approach solves.
The Airborne Collision Avoidance System X (ACAS X)~\cite{kochenderfer15} is a family of collision avoidance systems for both manned and unmanned aircraft. 
ACAS Xu, the variant for unmanned aircraft, is implemented as a large (2GB) numeric lookup table mapping the physical state of the aircraft and a neighboring object (an \emph{intruder}) to horizontal maneuver advisories. 
The lookup table is indexed on the distance ($\rho$) between the aircraft and the intruder, the relative angle ($\theta$) from the aircraft to the intruder, the angle ($\psi$) from the intruder's heading to the
aircraft's heading, the speed of the aircraft ($v_\textit{own}$), and of the
intruder ($v_\textit{int}$), and the time ($\tau$) until loss of vertical separation.
The possible advisories are either that no change is needed (or clear-of-conflict, COC), that the aircraft should steer weakly to the left, weakly to the right, strongly to the left, or strongly to the right.

As the table is too large for many unmanned avionics systems, \cite{julian2019acas} proposed the use of neural networks as a compressed, functional representation of the lookup table.
The networks proposed by \cite{julian2019acas} are functions $f : \mathbb{R}^5 \to \mathbb{R}^5$; the value $\tau$ is discretized and 45 different neural networks are constructed, one for each combination of the previous advisory ($a_\textit{prev}$) and discretized value of $\tau$.
Note that while the neural representation of the lookup table is an effective way to encode it on resource-constrained avionics systems, they are necessarily an approximation of the desired functionality, and may thus introduce unsafe behavior~\cite{katz17,sotoudeh2021provable,mirman2018differentiable,lin20,wu2020parallelization}.
To address this, \cite{katz17} proposed 10 safety properties, which capture requirements such as, ``If the intruder is directly ahead and is moving towards the ownship, the score for \emph{COC} will \emph{not} be maximal.''
Our goal is to construct networks that are guaranteed to satisfy specifications like these.

\subsection{Problem Definition}
\label{sec:properties:defn}

Definition~\ref{def:safe-ordering} presents the safe-ordering constraints that we consider throughout the rest of the paper.
Intuitively, they correspond to constraints on the relative ordering of a network's output values (a postcondition) with a predicate on the corresponding input (a precondition).
As we will see in later sections, the precondition does not need to belong to a particular theory, and need only come with an effective procedure for deciding new instances.

\begin{definition}[Safe ordering constraint]
\label{def:safe-ordering}
Given a neural network $f : \mathbb{R}^n \to \mathbb{R}^m$, a safe-ordering constraint, $\phi = \braket{P,Q}$, is a precondition, $P$, consisting of a decidable proposition over $\mathbb{R}^n$, and a postcondition, $Q$, given as a Boolean combination of order relations between the real components of $\mathbb{R}^m$.
\[
\begin{array}{lcrl}
\text{precondition} & \pre & := & \text{decidable proposition} \\ 
\text{ordering literal} & \pred & := & y_i < y_j~ (0 \le i, j < m) \\
\text{ordering constraint} & \post & := & \pred ~|~ \post \land \post ~|~ \post \lor \post \\
\text{safe-ordering constraint} & \phi & := & \braket{P,Q} \\ 
\text{set of constraints} &\Phi & := & \cdot ~|~ \phi, \Phi
\end{array}\vspace{0.5em}
\]
Assuming a function, \ls`eval $P$ : $\mathbb{R}^n \to$ bool`, that decides $P$ given $x \in \mathbb{R}^n$, notated as $P(x)$, and a similar \ls`eval` function for $Q$, we say $f$ satisfies safe-ordering constraint $\phi$ at $x$ iff $P(x) \implies Q(f(x))$.
	We use the shorthand $\phi(x, f(x))$ to denote this; and given a set of constraints $\Phi$, we write $\Phi(x, f(x))$ to denote $\forall \phi\in\Phi~.~\phi(x,f(x))$ and $\Phi(x)$ to denote $\bigwedge_{\braket{P_i,Q_i} \in \Phi~|~P_i(x)} Q_i$.
\end{definition}

Two points about our definition of safe-ordering constraints bear mentioning.
First, although postconditions are evaluated using the inequality relation from real arithmetic, we assume that $\forall x ~.~ i \ne j \implies f_i(x) \ne f_j(x)$, and thus specifically exclude equality comparisons between the output components.
This is a realistic assumption in nearly all practical settings, and in cases where it does not hold, can be resolved with arbitrary tie-breaking protocols that perturb $f(x)$ to remove any equalities. 
Second, we omit explicit negation from our syntax, as it can be achieved by swapping the positions of the affected order relations; i.e., $\lnot(y_i < y_j)$ is just $y_j < y_i$, as we exclude the possibility that $y_i = y_j$.

Sections~\ref{sec:eval:cifar} and \ref{sec:eval:synthetic} provide several concrete examples of safe-ordering constraints.
Example~\ref{ex:acas-property} revisits the safety specification for ACAS Xu that was discussed in the previous section. Notice that
this specification is an instance of the situation where \emph{safety need not imply accuracy},
since it does not specify what category \emph{should} be maximal;
that choice must be learned from the training data.

\begin{example}[Safety need not imply accuracy]
\label{ex:acas-property}
Recall the specification described earlier: ``If the intruder is directly ahead and is moving towards the ownship, the score for \emph{COC} will not be maximal.'' 
This is a safe-ordering constraint $\braket{P,Q}$, where the precondition $P$ is captured as a linear real arithmetic formula given by \cite{katz17}:
\begin{align*}
P \equiv~& 1500 \le \rho \le 1800 ~\land~ -0.06 \le \theta \le 0.06 ~\land~ \psi \ge 3.10 \\
	&  ~\land~ v_{\textit{own}} \ge 980 ~\land~v_{\textit{int}} \ge 960 \\
Q \equiv~& y_0 < y_1 ~\lor~ y_0 < y_2 ~\lor~ y_0 < y_3 ~\lor~ y_0 < y_4
\end{align*}

\noindent
In fact, nine of the ten specifications proposed by \cite{katz17} are safe-ordering constraints.
The single exception has a postcondition that places a constant lower-bound on $y_0$, i.e., a constraint on the logit value.
We do not consider such constraints because the exact logit values are often less meaningful than the class order, given the approximate nature of neural networks and the fact that logit values are not typically calibrated.
Moreover, the logit values of the network can be freely scaled without impacting the network's behavior as a classifier.
\end{example}

Given a set of safe-ordering constraints, $\Phi$, our goal is to obtain a neural network that satisfies $\Phi$ everywhere.
In later sections, we show how to accomplish this by describing the construction of a \emph{self-correcting transformer} (Definition~\ref{def:sr-transformer}) that takes an existing, possibly unsafe network, and produces a related model that 
satisfies $\properties$ at all points.
While in practice, a meaningful, well-defined specification $\properties$ should be satisfiable for all inputs, our generic formulation of safe-ordering constraints in Definition~\ref{def:safe-ordering} does not enforce this restriction;
we can, for instance, let $\Phi := \braket{\top, y_0 < y_1}, \braket{\top, y_1 < y_0}$.
To account for this, we lift predicates $\phi$ to operate on $\mathbb{R}^m \cup \{\bot\}$, where $\phi(x, \bot)$ is considered valid for all $x$.

\begin{definition}[Self-correcting transformer]
\label{def:sr-transformer}
A self-correcting transformer, $SC : \Phi \to \left(\mathbb{R}^n \to \mathbb{R}^m\right) \to \left(\mathbb{R}^n \to \left(\mathbb{R}^m \cup \{\bot\}\right)\right)$, is a function that, given a set of safe-ordering constraints, $\properties$, and a neural network, $f : \mathbb{R}^n \to \mathbb{R}^m$, produces a network, denoted as $\arnet : \mathbb{R}^n \to \left(\mathbb{R}^m \cup \{\bot\}\right)$, that satisfies the following properties:
\begin{enumerate}[label=(\roman*)]
\item \label{thm:safety} 
	\emph{Safety}: $\forall x~.~(~\exists y.~\Phi(x,y)~) \implies \Phi(x, \arnet(x))$
	\vspace{0.5em}
\item \label{thm:forewarning} 
	\emph{Forewarning}: $\forall x~.~(~\arnet(x) = \bot ~\Longleftrightarrow~ \forall y~.~\lnot\Phi(x,y)~)$
	\vspace{0.25em}
\end{enumerate}

\noindent
In other words, $\arnet = SC(\Phi)(f)$ is safe with respect to $\Phi$ and produces a non-$\bot$ output wherever $\properties(x)$ is satisfiable.
We refer to the output of $SC$, $\arnet$, as a self-correcting network (\arnet).
\end{definition}

Definition~\ref{def:sr-transformer}\ref{thm:safety} captures the essence of the problem that we aim to solve, requiring that the self-correcting network make changes to its output according to $\Phi$.
While allowing it to abstain from prediction by outputting $\perp$ may appear to relax the underlying problem, note that this is only allowed in cases where $\Phi$ cannot be satisfied on $x$:
definition~\ref{def:sr-transformer}\ref{thm:forewarning} is an equivalence that precludes trivial solutions such as $\arnet := \lambda x.\bot$.
However, it still allows abstention in exactly the cases where it is needed for principled reasons.
A set of safe-ordering constraints  may be mutually satisfiable almost everywhere, except in some places; for example: $\Phi := \braket{x \le 0.5, y_0 < y_1}, \braket{x \ge 0.5, y_1 < y_0}$.
In this case, $\arnet$ can abstain at $x = 0.5$, and everywhere else must produce outputs in $\mathbb{R}^m$ obeying $\Phi$.

While the properties required by Definition~\ref{def:sr-transformer} are sufficient to ensure a non-trivial, safe-by-construction neural network,
in practice, we aim to apply $SC(\Phi)$, which we will write as $SC^\Phi$, to models that \emph{already} perform well on observed test cases, but that still require a safety guarantee.
Thus, we wish to correct network outputs without interfering with the existing network behavior when possible, a property we call \emph{transparency} (Property~\ref{thm:transparency}).

\begin{Property}[Transparency]
\label{thm:transparency}
Let $SC : \Phi \to \left(\mathbb{R}^n \to \mathbb{R}^m\right) \to  \left(\mathbb{R}^n \to \left(\mathbb{R}^m \cup \{\bot\}\right)\right)$ be a self-correcting transformer.
We say that $SC$ satisfies transparency if
\begin{align*}
	&\forall \Phi~.~\forall f : \mathbb{R}^n \to \mathbb{R}^m ~.~\forall x\in\mathbb{R}^n~.~ \\
	&\left(\exists y.~\Phi(x, y) ~\wedge~ \argmax_i\{y_i\} = F(x)\right) \implies F^\Phi(x) = F(x)
\end{align*}
	where $F^\Phi(x) := \bot~\text{if}~f^\Phi(x) = \bot~\text{else}~\argmax_i\{f^\Phi_i(x)\}$.
\noindent
In other words, $SC$ always produces an \arnet, $\arnet$, for which the top class derived from the safe output vectors of $\arnet$ agrees with the top class of the original model whenever possible.
\end{Property}

Property~\ref{thm:transparency} leads to a useful result, namely that whenever $\Phi$ is consistent with \emph{accurate} predictions,
then the classifier obtained from $SC^\Phi(f)$ is at least as accurate as $F$ (Theorem~\ref{thm:preservation}). 
Formally, we characterize accuracy in terms of agreement with an oracle classifier $F^O$
that ``knows'' the correct class for each input, so that $F$ is accurate on $x$ if and only if $F(x) = F^O(x)$.
We note that accuracy is often defined with respect to a \emph{distribution} of labeled points rather than an oracle; however our formulation captures the key fact that Theorem~\ref{thm:preservation} holds regardless of how the data are distributed.

\begin{theorem}[Accuracy Preservation]
\label{thm:preservation}
Given a neural network, $f : \mathbb{R}^n \to \mathbb{R}^m$, and set of constraints, $\properties$,
let $f^\Phi := SC^\Phi(f)$
and let $F^O: \mathbb{R}^n \to [m]$ be the oracle classifier.
Assume that $SC$ satisfies transparency.
Further, assume that accuracy is consistent with safety, i.e.,
$$
\forall x\in\mathbb{R}^n~.~ \exists y~.~\Phi(x, y) ~\wedge~ \argmax_i\{y_i\} = F^O(x).
$$
Then,
$$
\forall x \in \mathbb{R}^n~.~ F(x) = F^O(x) \implies F^\Phi(x) = F^O(x)
$$
\end{theorem}

One subtle point to note is that even when $\Phi$ is consistent with accurate predictions, it is possible for a network to be accurate 
yet unsafe at an input. Example \ref{ex:acas-property2} describes such a situation. Our formulation of Property \ref{thm:transparency} is carefully designed to ensure accuracy preservation even in such scenarios.

\begin{example}[Accuracy need not imply safety]
\label{ex:acas-property2}
Consider the property $\phi_2$ proposed for ACAS Xu by \cite{katz17} which says: ``If  the  intruder  is  distant  and  is  significantly  slower  than  the ownship, the score of the COC advisory should never be minimal.''
This safe-ordering constraint is applicable for all networks that correspond to $a_\textit{prev} \neq$ COC and is concretely written as follows:
\begin{align*}
P &\equiv \rho \ge 55947.691 ~\land~ v_{\textit{own}} \ge 1145 ~\land~ v_{\textit{int}} \le 60 \\
Q &\equiv y_1 < y_0 ~\lor~ y_2 < y_0 ~\lor~ y_3 < y_0 ~\lor~ y_4 < y_0
\end{align*}
\noindent
For some $x$ such that $P(x)$ is true, let us assume that $F^O(x) = 1$ and for a network $f$, $f(x) = [100, 900, 300, 140, 500]$, so that $F(x) = 1$. 
Then, $f$ is accurate at $x$, but the COC advisory receives the minimal score, meaning $f$ is unsafe at $x$ with respect to $\phi_2$.
If the transformer $SR$ satisfies Property \ref{thm:transparency}, then by Theorem \ref{thm:preservation}, $f^{\phi_2}$ is guaranteed to be accurate as well as safe at $x$, since $\phi_2$ is consistent with accuracy here (as $\phi_2$ does not preclude class 1 from being maximal).
\end{example}


\section{Self-correcting Transformer}
\label{sec:repair}

We describe our self-correcting transformer, \Fselfrepair. 
We begin with a high-level overview of the approach (Section~\ref{sec:repair:overview}), and provide algorithmic details in Section~\ref{sec:repair:details}.  We then provide proofs (Section~\ref{sec:repair:correctness}) and complexity analysis (Section~\ref{sec:repair:complexity}).

\subsection{Overview}
\label{sec:repair:overview}

Our self-correcting transformer, \Fselfrepair, leverages the fact that whenever a safe-ordering constraint is satisfiable at a point, it is possible to bring the network into compliance.
Neural networks are typically constructed by composing a sequence of layers; we thus compose an additional \emph{self-correction layer} that operates on the original network's output, and produces a result that will serve as the transformed network's new output.
This is reflected in the \Fselfrepair routine in Algorithm~\ref{alg:self_repair}.
The original network, $f$, executes normally, and the self-correction layer subsequently takes both the input $x$ (to facilitate checking the preconditions of $\properties$) and $y:= f(x)$, from which it either abstains (outputs $\bot$) or produces an output that is guaranteed to satisfy $\Phi$.

The high-level workflow of the self-correction layer, \Frepair, 
proceeds as follows.
The layer starts by checking the input $x$ against each of the preconditions, and derives an \emph{active postcondition}.
This is then passed to a solver, which attempts to find the set of orderings that are consistent with the active postcondition.
If no such ordering exists, i.e., if the active postcondition is unsatisfiable, then the layer abstains with $\bot$.
Otherwise, the layer minimally permutes the indices of the original output vector in order to satisfy the active postcondition while ensuring transparency (Property~\ref{thm:transparency}).

\subsection{Algorithmic Details of SC-Layer}
\label{sec:repair:details}

The core logic of our approach is handled by a self-correction layer, or SC-Layer, that is appended to the original model, and dynamically ensures its outputs satisfy the requisite safety specifications.
The procedure followed by this layer,
\Frepair (shown in Algorithm~\ref{alg:self_repair}), first checks if the input $x$ and output $y$ of the base network already satisfy $\properties$ (line~\ref{line:ifprop}).
If they do, no correction is necessary and the repaired network $f^\properties$ can safely return $y$. 
Otherwise, \Frepair attempts to find a satisfiable ordering constraint that entails the relevant postconditions in $\properties$ (line~\ref{line:fsolve}). 
\Fsolve either returns such a term $q$ that consists of a conjunction of ordering literals $y_i < y_j$, or returns $\bot$ whenever no such $q$ exists.
When \Fsolve returns $\bot$, then $\Frepair$ does as well (lines~\ref{line:botif}-\ref{line:botret}).
Otherwise, the constraint identified by \Fsolve is used to correct the network's output (line~\ref{line:frepair}),
where \Freorder permutes the logit values in $y$ to arrive at a vector that satisfies $q$. 
Note that because $q$ is satisfiable, it is always possible to find 
a satisfying solution by simply permuting $y$, because the specific real values are irrelevant, and only their order matters (see Section~\ref{sec:repair:correctness}).

\begin{algorithm}[t!]
\small
\vspace{0.5em}
\KwIn{A set of safety properties, $\properties$ and a network, $f : \mathbb{R}^n \to \mathbb{R}^m$}
\KwOut{A network, $f^\properties : \mathbb{R}^n \to \mathbb{R}^m \cup \{\bot\}$}
\vspace{0.5em}
\Fn{\Fselfrepair{$\properties ~~,~~ f$}}{
	$f^\Phi ~:=~ \lambda~x.\Frepair(\properties,x,f(x))$\;
	\textbf{return} $f^\Phi$\;
}
\vspace{0.25em}
\Fn{\Frepair{$\properties ~~,~~ x ~~,~~ y$}}{
	\eIf{$\properties(x, y)$}{ \label{line:ifprop}
		\textbf{return} $y$\;
	}{
		$q ~:=~\Fsolve(\Phi,x,y)$\; \label{line:fsolve}
		\eIf{$q = \bot$}{  \label{line:botif}
			\textbf{return} $\bot$\; \label{line:botret}
		}{
			$y' ~:=~\Freorder(q,y)$\; \label{line:frepair}
			\textbf{return} $y'$\;
		}
	}
}
\caption{Self-correcting transformer}
\label{alg:self_repair}
\end{algorithm}

\subsubsection{Finding Satisfiable Constraints}
\label{sec:repair:solve}

Algorithm~\ref{alg:repair:solve} illustrates the \linebreak \Fsolve procedure.
Recall that the goal is to identify a conjunction of ordering literals $q$ that implies the \emph{relevant} postconditions in $\properties$ at the given input $x$.
More precisely, this means that for each precondition $P_i$ satisfied by $x$, the corresponding postcondition $Q_i$ is implied by $q$.
This is sufficient to ensure that any model $y'$ of $q$ will satisfy \properties at $x$; i.e., $q(y') \implies \Phi(x, y')$.

To accomplish this, \Fsolve first evaluates each precondition, and obtains (line~\ref{line:todnf}) a disjunctive normal form (DNF), $Q_x$, of the \emph{active postcondition}, defined by
$\Ffilter(\Phi,x)~:=~ \bigwedge_{\braket{P_i,Q_i} \in \Phi~|~P_i(x)} Q_i$.
In practice, we implement a lazy version of \Ftodnf that generates disjuncts as needed (see Section~\ref{sec:impl}),
as this step may be a bottleneck, and we only need to process each clause individually.
At this point, \Fsolve could proceed directly, checking the satisfiability of each disjunct in $Q_x$, and returning the first satisfiable one it encounters.
This would be correct, but as we wish to satisfy transparency (Property~\ref{thm:transparency}), we first construct an ordered list of the terms in $Q_x$ which prioritizes constraints that maintain the maximal position of the original prediction, $\argmax(y)$ (\Fprioritize, line~\ref{line:prioritize}). 
Property~\ref{prop:prioritize} formalizes the behavior required of \Fprioritize.

\begin{Property}[Prioritize]
\label{prop:prioritize}
Given $y \in \mathbb{R}^m$ and a list of conjunctive ordering constraints $\overline{Q}$,
the result of $\Fprioritize(\overline{Q}, y)$ is a reordered
list $\overline{Q}' = [\ldots, q_i, \ldots]$ such that:
\begin{align*}
	\label{inv:prioritize}
	&\forall~ 0 \le i, j < |\overline{Q}|~ .~ \argmax_i\{y_i\} \in \rootg(\Fordergraph(q_i)) \\
	&\wedge \argmax_i\{y_i\} \not\in \rootg(\Fordergraph(q_j)) \implies i < j 
\end{align*}
where $\rootg(G)$ denotes the root nodes of the directed graph $G$.
\end{Property}

The \Fissat procedure (invoked on line~\ref{line:issatcall}, also shown in Algorithm~\ref{alg:repair:solve}) checks the satisfiability of a conjunctive ordering constraint.
It is based on an encoding of $q$ as a directed graph, embodied in \Fordergraph (lines~\ref{line:orderstart}-\ref{line:orderend}), where each component index of $y$ corresponds to a node, and there is a directed edge from $i$ to $j$ if the literal $y_j < y_i$ appears in $q$.
A constraint $q$ is satisfiable if and only if $\Fordergraph(q)$ contains no cycles (lines~\ref{line:issatstart}-\ref{line:issatend})~\cite{graphissat}.
Informally, acyclicity is necessary and sufficient for satisfiability because the directed edges encode immediate ordering requirements,
and by transitivity, a cycle involving $i$ entails that $y_i < y_i$.

\begin{algorithm}[t!]
\small
\vspace{0.5em}
\KwIn{A set of safe-ordering constraints, $\properties$, a vector $x : \mathbb{R}^n$, and a vector $y : \mathbb{R}^m$}
\KwOut{Satisfiable ordering constraint, $q$}
\vspace{0.5em}
\Fn{\Fordergraph{q}} { \label{line:orderstart}
	$V~:=~[m]$\;
	$E~:=~\{(i,j) : y_j < y_i \in q\}$\;
	\textbf{return} $(V,E)$ \label{line:orderend}
}
\vspace{0.5em}
\Fn{\Fissat{q}} { \label{line:issatstart}
	$g~:=~\Fordergraph{q}$\;
	\textbf{return} $\lnot\Fcontainscycle(V,E)$ \label{line:issatend}
}
\vspace{0.5em}
\Fn{\Fsolve{$\properties ~~,~~ x ~~,~~ y$}}{
	$Q_x ~:=~\Ftodnf(\Ffilter(\properties,x))$\; \label{line:todnf}
  $Q_p ~:=~\Fprioritize(Q_x, y)$\; \label{line:prioritize}
	\ForEach{$q_i \in Q_p$}
	{
		\If{$\Fissat(q_i)$}{ \label{line:issatcall}
			\textbf{return} $q_i$
		}

	}
	\textbf{return} $\bot$\;
}
\caption{Finding a satisfiable ordering constraint from safe-ordering constraints $\properties$}
\label{alg:repair:solve}
\end{algorithm}

\subsubsection{Correcting Violations}
\label{sec:repair:reorder}

\begin{algorithm}[t!]
\small
\vspace{0.5em}
\KwIn{Satisfiable ordering constraint $q$, a vector $y : \mathbb{R}^m$}
\KwOut{A vector $y' : \mathbb{R}^m$}
\vspace{0.5em}
	\Fn{\Freorder{$q ~~,~~ y$}}{
	$\pi~:=~\Ftopsort(\Fordergraph(q),y)$\; \label{line:topsortcall}
	$y^s ~:=~\Fdescsort(y)$\;
  $\forall~j \in [m] ~~.~~ y'_j ~:=~ y^s_{\pi(j)}$\; \label{line:permute}
  \textbf{return} $y'$\;
}
\caption{Correction procedure for safe-ordering constraints}
\label{alg:repair:reorder}
\end{algorithm}

Algorithm \ref{alg:repair:reorder} describes the \Freorder procedure, used to ensure the outputs of the SC-Layer satisfy safety. 
The inputs to \Freorder are a satisfiable ordering constraint $q$, and the output of the original network $y := f(x)$. 
The goal is to permute $y$ such that the result $y'$ satisfies $q$, without violating transparency. 
Our approach is based on \Fordergraph, the same directed-graph encoding used by \Fissat.
It uses a stable topological sort of the graph encoding of $q$ to construct a total order over the indices of $y$ that is consistent with the partial ordering implied by $q$ (line~\ref{line:topsortcall}). \Ftopsort returns a permutation $\pi$, a function that maps indices in $y$ to their rank (or position) in the total order.
Formally, \Ftopsort takes as argument a graph $G = (V, E)$, and returns
$\pi$ such that Equation~\ref{eqn:topsort1} holds.
\begin{equation}
\label{eqn:topsort1}
\forall i,j\in V~.~(i, j) \in E \implies \pi(i) < \pi(j)
\end{equation}

Informally, if the edge $(i, j)$ is in the
graph, then $i$ occurs before $j$ in the ordering.
In general, many total orderings may be consistent, but in order to guarantee transparency, \Ftopsort also needs to ensure the following invariant (Property~\ref{eqn:topsort2}),
capturing that the maximal index is listed first in the total order if possible.
\begin{Property}
\label{eqn:topsort2}
Given a graph, $G = (V, E)$, and $y \in \mathbb{R}^m$,
the result $\pi$ of \linebreak $\Ftopsort(G, y)$ satisfies\vspace{-0.5em}
$$
\argmax_i\{y_i\} \in \rootg({G}) \implies \pi\left(\argmax_i\{y_i\}\right) = 0
$$
where $\rootg(G)$ denotes the root nodes of the directed graph $G$.
\end{Property}

In other words, the topological sort preserves the network's original prediction when doing so is consistent with $q$.
Then, by sorting $y$ in descending order, the sorted vector $y^s$ can be used to construct the final output of \Freorder, $y'$. 
For any index $i$,
we simply set $y'_i$ to the $\pi(i)^{th}$ component of $y^s$, since $\pi(i)$ gives the desired rank of the $i^{th}$ logit value and 
components in $y^s$ are sorted according to the component values (line~\ref{line:permute}).
Example~\ref{ex:repair} shows an example of the complete \Freorder procedure. 


\begin{example}[Self-correct]
\label{ex:repair}
We refer again to the safety properties introduced for ACAS Xu~\cite{katz17}.
The postcondition of property $\phi_2$ states that the logit score for class 0 (COC) is not minimal, which can be written as the following ordering constraint:
$$
Q \equiv y_1 < y_0 ~\lor~ y_2 < y_0 ~\lor~ y_3 < y_0 ~\lor~ y_4 < y_0
$$
Suppose that for some input $x\in\mathbb{R}^n$, the active postcondition is equivalent to $Q$, and that $y = [100, 900, 300, 140, 500]$.
Further, suppose that \Fsolve has returned $q := y_2 < y_0$, corresponding to the second disjunct of $Q$ (satisfying $q \implies Q$).
We then take the following steps according to $\Freorder(q, y)$:
\begin{itemize}
\item
	First we let $\pi := \Ftopsort(\Fordergraph(q), y)$. We note that all vertices of the graph representation of $q$ are roots except for $j = 2$, which has $j = 0$ as its parent.
	We observe that $\argmax_i\{y_i\} = 1$, which corresponds to a root node; thus by Property~\ref{eqn:topsort2}, $\pi(1) = 0$.
	Moreover, by our ordering constraint, we also have that $\pi(0) < \pi(2)$.
	Thus, the ordering $\pi$ where $\pi(0) = 2$, $\pi(1) = 0$, $\pi(2) = 3$, $\pi(3) = 4$, and $\pi(4) = 1$ is a possible result of \Ftopsort, which we will assume for this example.
\item
	Next we obtain by a descending sort that $y^s = [900, 500, 300, \allowbreak 140, 100]$.
\item
	Finally we obtain $y'$ by indexing $y^s$ by the inverse of $\pi$, i.e.,  $y'_j = y^s_{\pi(j)}$.
	This gives us  $y'_0 = y^s_2 = 300$, $y'_1 = y^s_0 = 900$, $y'_2 = y^s_3 = 140$, $y'_3 = y^s_4 = 100$, and $y'_4 = y^s_1 = 500$, resulting in a final output of $y' = [300, 900, 140, 100, 500]$, which  (i) satisfies $Q$, and (ii) preserves the prediction of class 1.
\end{itemize}
\end{example}

\subsection{Key Properties}
\label{sec:repair:correctness}

We now provide a brief argument that our \Fselfrepair procedure satisfies two key properties; namely (1) \Fselfrepair is a self-correcting transformer (Definition~\ref{def:sr-transformer})---i.e., it guarantees that the corrected output will always satisfy the requisite safety properties, unless they are \emph{unsatisfiable}, in which case it returns $\bot$---and (2) \Fselfrepair is transparent (Property~\ref{thm:transparency})---i.e., it does not modify the predicted class (the class with the maximal logit value) unless it is absolutely necessary for safety.
Full proofs appear in Appendix~\ref{app:proofs}.

\begin{theorem}[\Fselfrepair is a self-correcting transformer]
\label{thm:sr-sound}
\Fselfrepair (Algorithm~\ref{alg:self_repair}) satisfies conditions \emph{(i)} and \emph{(ii)} of Definition~\ref{def:sr-transformer}.
\end{theorem}

\noindent
This follows from the construction of \Fselfrepair, and relies on a few key properties of \Fsolve and \Freorder. 
First, whenever \Fsolve returns $\bot$, the set of safety constraints, \properties, is \emph{unsatisfiable} on the given input.
Second, whenever \Fsolve returns some $q\neq\bot$, then $q$ is satisfiable on the given input.
Finally, when $q$ is satisfiable, \Freorder always modifies the output such that it satisfies $q$.
Together, these imply Theorem~\ref{thm:sr-sound}.

In addition to ensuring safe-ordering, \Fselfrepair is transparent (Theorem~\ref{thm:sr-transparency}), which recall is a precondition for the accuracy preservation property stated in Theorem~\ref{thm:preservation}.

\begin{theorem}[Transparency of \Fselfrepair]
\label{thm:sr-transparency}
\Fselfrepair, the self-correcting transformer described in Algorithm~\ref{alg:self_repair}, satisfies Property~\ref{thm:transparency}.
\end{theorem}

\noindent
Clearly, on points where the model naturally satisfies the safety properties, no changes to the output are made and \Fselfrepair is transparent.
Otherwise, we rely on a few key details of our construction to achieve transparency.

We begin with the observation that whenever the network's predicted top class is a root of the graph encoding of a satisfiable postcondition, $q$, there exists an output that satisfies $q$ while preserving the predicted top class.
Intuitively, this follows because the partial ordering admits \emph{any} of the root nodes to appear first in the total ordering.

With this in mind, we recall that \Fsolve searches potential solutions according to \Fprioritize, which prefers all solutions in which the predicted top class appears as a root node over any in which it does not.
Thus, \Fprioritize will return a solution that is consistent with preserving the network's original predicted top class whenever possible.

Finally, we design our topological sort to be ``stable,'' such that, among other things, the network's original top prediction will appear first in the total ordering whenever it appears as a root node.
More details on our topological sort algorithm and the properties it possesses are given in Section~\ref{sec:impl:stable_topological_sort}.

\subsection{Complexity}
\label{sec:repair:complexity}

Given a neural network $f:\mathbb{R}^n \to \mathbb{R}^m$, we define the input size as $n$ and output size as $m$. Also, assuming that
the postconditions $Q_i$ for all $\braket{P_i,Q_i} \in \Phi$ are expressed in DNF, we define the size $p_i$ of a constraint as the number of disjuncts in $Q_i$ and define $\alpha := |\properties|$, i.e., the number of properties in $\properties$.
Then, the worst-case computational complexity of \arlayer is given by Equation~\ref{eq:complexity},
where $O(log(m))$ is the complexity of \Fcontainscycle, $O(mlog(m))$ is the complexity of \Ftopsort, and $\prod_{i=1}^{\alpha}{p_i}$ is the maximum number
of disjuncts possible in $Q_x$ if the postconditions $Q_i$ are in DNF.
\begin{equation}
\label{eq:complexity}
O\left(log(m)\prod_{i=1}^{\alpha}{p_i} + m\,log(m)\right)
\end{equation}

The complexity given by Equation~\ref{eq:complexity} is with respect to a cost model that treats matrix operations---e.g., matrix multiplication, associative row/column reductions---as constant-time primitives.
Crucially, note that the complexity does not depend on the size of the neural network $f$.

\subsection{Differentiability of SC-Layer}
\label{sec:repair:differentiability}
One interesting facet of our approach that remains largely unexplored is the differentiability of the \Frepair.
In principle, this opens the door to benefits that could be obtained by training against the corrections made by the \Frepair.
Though we found that a ``vanilla'' attempt to train with the \Frepair did not provide clear advantages over appending it to a model post-learning, we believe this remains an interesting future direction to explore.
It is conceivable that a careful approach to training \arnets could lead to safer and more accurate models, reducing the need for the run-time correction, as the network could learn to use the modifications made by the \Frepair to its advantage.
Furthermore, aspects of the algorithm, including, e.g., the heuristic used to prioritize the search for a satisfiable graph (see Section~\ref{sec:repair:solve}), could be parameterized and learned, potentially leading to both accuracy and performance benefits.



\section{Vectorizing Self-Correction}
\label{sec:impl}

Widely-used machine learning libraries, such as TensorFlow~\cite{abadi2016tensorflow}, simplify the implementation of parallelized, hardware-accelerated code by providing a collection of operations on multi-dimensional arrays of uniform type, called \emph{tensors}.
One can view such libraries as domain-specific languages that operate primarily over tensors, providing embarrassingly parallel operations like matrix multiplication and associative reduction, as well as non-parallel operations like iterative loops and sorting routines.
We use matrix-based algorithms implementing the core procedures used by \Frepair described in Section~\ref{sec:repair}.
As we will later see in Section~\ref{sec:eval}, taking advantage of these frameworks allows our implementation to introduce minimal overhead, typically fractions of milliseconds.
Additionally, it means that \Frepair can be automatically differentiated, making it fully compatible with training and fine-tuning.
One may use an SMT solver like Z3 to implement the procedures used by \Frepair but we found that making calls to an SMT solver significantly restricts the efficient use of GPUs.
Moreover, it is a useful heuristic for the corrected logits to prioritize the original ordering relationships. Encoding this heuristic would require an optimization variant of SMT (Max-SMT).
In contrast, our matrix-based algorithms efficiently calculate the corrected output while prioritizing the original class order.
We present the algorithmic details in Appendix~\ref{app:alg_details}.


\section{Evaluation}
\label{sec:eval}

We have shown that self-correcting networks (\arnets) provide safety to an existing network without affecting accuracy, as long as safety and accuracy are mutually consistent. This comes with no additional training cost, suggesting that the only potential downside of \arnets is the run-time overhead introduced by the \Frepair.
In this section, we present an empirical evaluation of our approach to demonstrate its scalability, and find that the run-time performance is not an issue in practice---overheads range from 0.2-0.8 milliseconds, and scale favorably with the size and complexity of constraints.

We explore the capability of our approach on a variety of domains, demonstrating its ability to solve previously studied safety-verification problems (Sections~\ref{sec:eval:acas} and \ref{sec:eval:collision}), and its ability to efficiently scale both (i) to large convolutional networks (Section~\ref{sec:eval:cifar}) and (ii) to arbitrary, complex safe-ordering constraints containing disjunctions and overlapping preconditions (Section~\ref{sec:eval:synthetic}).

We implemented our approach in Python, using TensorFlow to vectorize our \Frepair (Section~\ref{sec:impl}).
All experiments were run on an NVIDIA TITAN RTX GPU with 24 GB of RAM, and a 4.2GHz Intel Core i7-7700K with 32 GB of RAM.

\subsection{ACAS Xu}
\label{sec:eval:acas}

ACAS Xu~\cite{kochenderfer15} is a collision avoidance system that has been
frequently studied in the context of neural classifier safety verification~\cite{julian2019acas,katz17,lin20,singh19}.
Typically considered for this problem is a family of 45 networks proposed by \cite{julian2019acas}.
\cite{katz17} proposed 10 safety specifications for this family of networks, which have become standard for research on this problem.
We consider 9 of these specifications, which can be expressed as \emph{safe-ordering constraints} (Section~\ref{sec:properties:defn}).

Each of the 45 networks consists of six hidden dense layers of 50 neurons each.
Each network needs to satisfy some subset of the 10 safety constraints;
that is, more than one safety constraints may apply to each model, but not all safety constraints apply to each model.
A network is considered safe if it satisfies \emph{all} of the relevant safety constraints.
Among the 45 networks, \cite{katz17} reported that 9 networks were already safe after standard training, while 36 were unsafe, exhibiting safety constraint violations.

The data used to train the 45 networks is not publicly available; however, \cite{lin20} provide a synthetic test set for each network,
consisting of 5,000 points uniformly sampled from the specified state space and labeled using the respective network as an oracle.
We note that because this test set is labeled using the original models, the accuracy of each original model on this test set is necessarily 100\%.

\begin{table*}
	\centering
\begin{subfigure}{0.6\textwidth}
	\centering
	\resizebox{\textwidth}{!}{%
	\begin{tabular}{ll|cc}
	\toprule
	& \textit{method} & \textit{safe networks} & \textit{mean accuracy (\%)} \\
	\midrule
	\textit{36 unsafe nets} & original &          0 / 36  & 100.0 \\
	                        & ART      &         36 / 36  & 94.4 \\
	                        & SC-Net   & \textbf{36 / 36} & \textbf{100.0} \\
	\midrule
	\textit{9 safe nets}    & original & \textbf{9 / 9} & \textbf{100.0} \\
	                        & ART      &         9 / 9  &          94.3 \\
	                        & SC-Net   & \textbf{9 / 9} & \textbf{100.0} \\
	\bottomrule
	\end{tabular}}
	\vspace{.15em}
	\caption{}
	\label{tab:acas}
  \vspace{-2mm}
\end{subfigure}

\begin{subfigure}{0.4\textwidth}
	\centering
	\resizebox{\textwidth}{!}{%
	\begin{tabular}{l|cc}
	\toprule
	\textit{method} & \textit{constraints certified} & \textit{accuracy (\%)} \\
	\midrule
	original &         328 / 500  &         99.9 \\
	ART      &         481 / 500  &         96.8 \\
	SC-Net   & \textbf{500 / 500} & \textbf{99.9} \\
	\bottomrule
	\end{tabular}}
	\vspace{.15em}
	\caption{}
	\label{tab:collision}
\end{subfigure}\quad
\begin{subfigure}{0.4\textwidth}
	\centering
	\resizebox{\textwidth}{!}{%
	\begin{tabular}{l|c}
	\toprule
	\textit{dataset}      & \textit{overhead (ms)} \\
	\midrule
	ACAS Xu               & 0.26 \\
	Collision Detection   & 0.58 \\
	CIFAR-100 (small CNN) & 0.77 \\
	CIFAR-100 (ResNet-50) & 0.82 \\
	Synthetic             & 0.27 \\
	\bottomrule
	\end{tabular}}
	\vspace{.15em}
	\caption{}
	\label{tab:overhead}
\end{subfigure}
\caption{
	Safety certification results on the \textbf{(\subref{tab:acas})} ACAS Xu~\cite{julian2019acas} and \textbf{(\subref{tab:collision})} Collision Detection~\cite{ehlers2017formal} datasets.
	We compare the success rate and accuracy to that of ART~\cite{lin20}, a recent safe-by-construction approach.
	The original network is provided as a baseline.
	Best results are shown in bold.
	\textbf{(\subref{tab:overhead})} Absolute overhead introduced by the SC-Layer per input.
}\vspace{-3mm}
\label{tab:acas_and_collision}
\end{table*}

Table~\ref{tab:acas} presents the results of applying our \Fselfrepair transformer to each of the 45 provided networks.
In particular, we consider the number of networks for which safety can be guaranteed, and the accuracy of the resulting \arnet.
We compare our results to those using ART~\cite{lin20}, a recent approach to safe-by-construction learning. ART aims to learn neural networks that satisfy safety specifications expressed using linear real arithmetic constraints.
It updates the loss function to be minimized during learning by adding a term, referred to as the {\em correctness loss}, that measures the degree to which a neural network satisfies or violates the safety specification. A value of zero for the correctness loss ensures that the network is safe. However, there is no guarantee that learning will converge to zero correctness loss, and the resulting model may not be as accurate as one trained with conventional methods.

Because the safety constraints for each network are satisfiable on all points, Definition~\ref{def:sr-transformer} tells us that safety is guaranteed for all 45 \arnets.
In this case, we see that ART also manages to produce 45 safe networks after training; however we see that it comes at a cost of nearly 6 percentage points in accuracy, \emph{even on the networks that were already safe}.
Meanwhile, transparency (Property~\ref{thm:transparency}) tells us that \arnets will only see a decrease in accuracy relative to the original network when accuracy is in direct conflict with safety.
On the 9 original networks that were reported as safe, clearly no such conflict exists, and accordingly, we see that the corresponding \arnets achieve the same accuracy as the original networks (100\%).
On the 36 unsafe networks, we find again that the \arnets achieved 100\% accuracy.
In this case, it would have been possible that the \arnets would have achieved lower accuracy than the original networks, as some of the safety properties have the potential to conflict with accuracy.
For example, the postcondition of the property $\phi_8$ requires that the predicted maneuver advisory is either to continue straight (COC) or to turn weakly to the left.
Thus, correcting $\phi_8$ on inputs for which it is violated would necessarily change the network's prediction on those inputs; and, since the labels are derived from the original networks' predictions, this would lead to a drop in accuracy.
However, we find that none of the test points include violations of such constraints (even though such violations exist in the space generally~\cite{katz17}), as evidenced by the fact that the \arnet accuracy remained unchanged.

Table~\ref{tab:overhead} shows the average overhead introduced by applying SC to each of the ACAS Xu networks.
We see that the absolute overhead is only $\sim 0.25$ms per instance on average, accounting for less than an $8\times$ increase in prediction time.

\subsection{Collision Detection}
\label{sec:eval:collision}

The Collision Detection dataset~\cite{ehlers2017formal} provides another instance of a safety verification task that has been studied in the prior literature.
In this setting, a neural network controller is trained to predict whether two vehicles following curved paths at different speeds will collide.
As this is a binary decision task, the network contains two outputs, corresponding to the case of a collision and the case of no collision.
\cite{ehlers2017formal} proposes 500 safety properties for this task, corresponding to $\ell_\infty$ \emph{robustness regions} around 500 particular inputs;
i.e., property $\phi_i$ for $i\in \{1, \ldots, 500\}$ corresponds to a point, $x_i$, and radius, $\epsilon_i$, and is defined according to Equation~\ref{eq:robust_region}.
\begin{equation}
\label{eq:robust_region}
\phi_i(x, y) ~:=~ ||x_i - x||_\infty \leq \epsilon_i \implies y = F(x_i)
\end{equation}

Such specifications of \emph{local robustness} at fixed inputs can be represented as safe-ordering constraints, where the postcondition of $\phi_i$ is defined to be $y_0 > y_1$ if $F(x_i) = 0$ and $y_0 < y_1$ if $F(x_i) = 1$.

Table~\ref{tab:collision} presents the results of applying our \Fselfrepair transformer to the original network provided by \cite{ehlers2017formal}.
Similarly to before, we consider the number of constraints with respect to which safety can be guaranteed, and the accuracy of the resulting \arnet, comparing our results to those of ART.

We see in this case that ART was unable to guarantee safety for all 500 specifications.
Meanwhile, it resulted in a drop in accuracy of approximately 3 percentage points.
On the other hand, it is simple to check that the conjunction of all 500 safety constraints is satisfiable for all inputs; thus, Definition~\ref{def:sr-transformer} tells us that safety is guaranteed with respect to all properties.
Meanwhile \arnets impose no penalty on accuracy, as none of the test points violate the constraints.

Table~\ref{tab:overhead} shows the overhead introduced by applying \Fselfrepair to the collision detection model.
In absolute terms, we see the overhead is approximately half a millisecond per instance, accounting for under a $3\times$ increase in prediction time.

\subsection{Scaling to Larger Domains}
\label{sec:eval:cifar}

One major challenge for many approaches that attempt to verify network safety---particularly post-learning methods---is scalability to very large neural networks.
Such networks pose a problem for several reasons.
Many approaches analyze the parameters or intermediate neuron activations using algorithms that do not scale polynomially with the network size.
This is a practical problem, as large networks in use today contain hundreds of millions of parameters.
Furthermore, abstractions of the behavior of large networks may see compounding imprecision in large, deep networks.

Our approach, on the other hand, treats the network as a black-box and is therefore not sensitive to its specifics.
In this section we demonstrate that this is borne out in practice; namely the absolute overhead introduced by our \Frepair remains relatively stable even on very large networks.

For this, we consider a novel set of safety specifications for the CIFAR-100 image dataset~\cite{krizhevsky09}, a standard benchmark for object recognition tasks.
The CIFAR-100 dataset is comprised of 60,000 $32\times 32$ RGB images categorized into 100 different classes of objects, which are grouped into 20 superclasses of 5 classes each.
We propose a set of safe-ordering constraints that are reminiscent of a variant of \emph{top-$k$ accuracy} restricted to members of the same superclass, which has been studied recently in the context of certifying relational safety properties of neural networks~\cite{leino2021relaxing}.
More specifically, we require that if the network's prediction belongs to superclass $C_k$ then the top 5 logit outputs of the network must all belong to $C_k$.
Formally, there are 20 constraints, one for each superclass, where the constraint, $\phi_k$ for superclass $C_k$, for $k \in \{1,\ldots,20\}$, is defined according to Equation~\ref{eq:cifar100_property}. Notice that with respect to these constraints, a standard trained network can be accurate, yet unsafe, even without accuracy and safety being mutually inconsistent.
\begin{equation}
\label{eq:cifar100_property}
\phi_k(x, y) ~:=~ F(x) \in C_k \implies \bigwedge_{i,j ~|~ i\in C_k,~ j\notin C_k}y_j < y_i
\end{equation}

As an example application requiring this specification, consider a client of the classifier that averages the logit values over a number of samples for classes appearing in top-5 positions
and chooses the class with the highest average logit value (due to imperfect sensor information – a scenario similar to the ACAS Xu example).
A reasonable specification is to require that the chosen class shares its superclass with at least one of the top-1 predictions. We can ensure that this specification is satisfied by enforcing Equation~\ref{eq:cifar100_property}.

Table~\ref{tab:overhead} shows the overhead introduced by applying \Fselfrepair, with respect to these properties, to two different networks trained on CIFAR-100.
The first is a convolutional neural network (CNN) that is much smaller than is typically used for vision tasks, containing approximately 1 million parameters.
The second is a standard residual network architecture, ResNet-50~\cite{he16resnet}, with approximately 24 million parameters.

In absolute terms, we see that both networks incur less than 1ms of overhead per instance relative to the original model (0.77ms and 0.82ms, respectively), making \Fselfrepair a practical option for providing safety in both networks.
Moreover, the absolute overhead varies only by about 5\% between the two networks, suggesting that the overhead is not sensitive to the size of the network.
This overhead accounts for approximately a $12\times$ increase in prediction time on the CNN.
Meanwhile, the overhead on the ResNet-50 accounts for only a $6\times$ increase in prediction time relative to the original model.
The ResNet-50 is a much larger and deeper network; thus its baseline prediction time is longer, so the overhead introduced by the \Frepair accounts for a smaller fraction of the total computation time.
In this sense, our \Fselfrepair transformer becomes relatively \emph{less} expensive on larger networks.

Interestingly, we found that the original network violated the safety constraints on approximately \emph{98\% of its inputs}, suggesting that obtaining a violation-free network without \Fselfrepair might prove particularly challenging.
Meanwhile, the \arnet eliminated \emph{all} violations, with \emph{no cost to accuracy}, and less than 1ms in overhead per instance.

\subsection{Handling Arbitrary, Complex Constraints}
\label{sec:eval:synthetic}

\begin{figure}[t]
\begin{subfigure}{0.5\columnwidth}
\resizebox{\textwidth}{!}{
\includegraphics{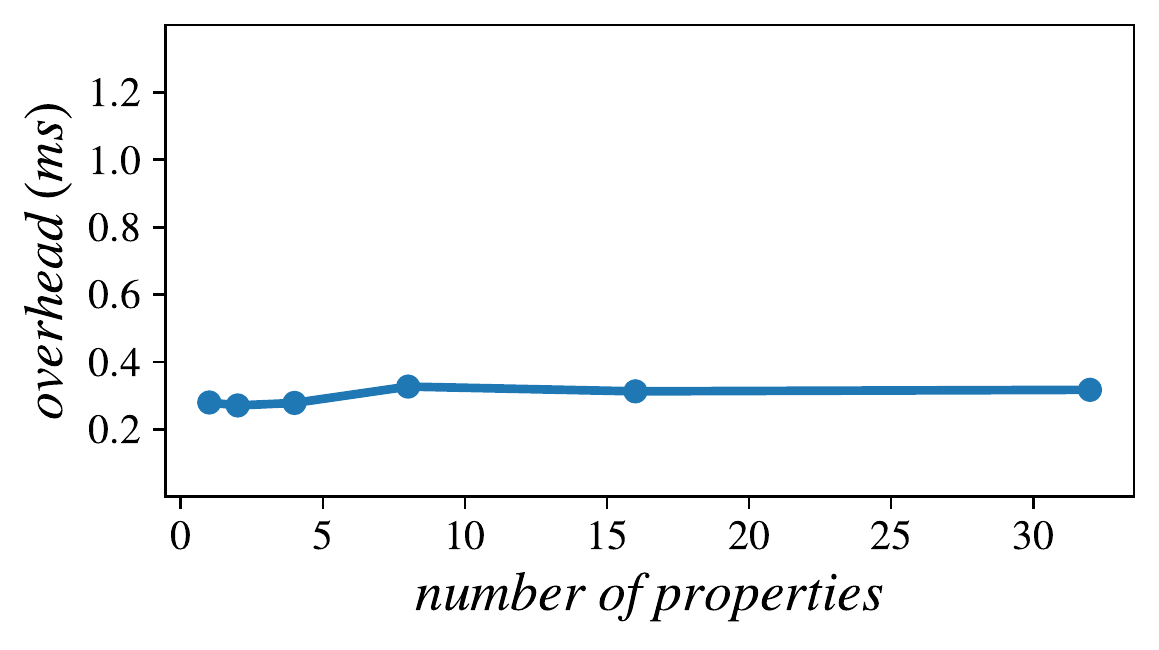}}
\end{subfigure}%
\begin{subfigure}{0.5\columnwidth}
\resizebox{\textwidth}{!}{
\includegraphics{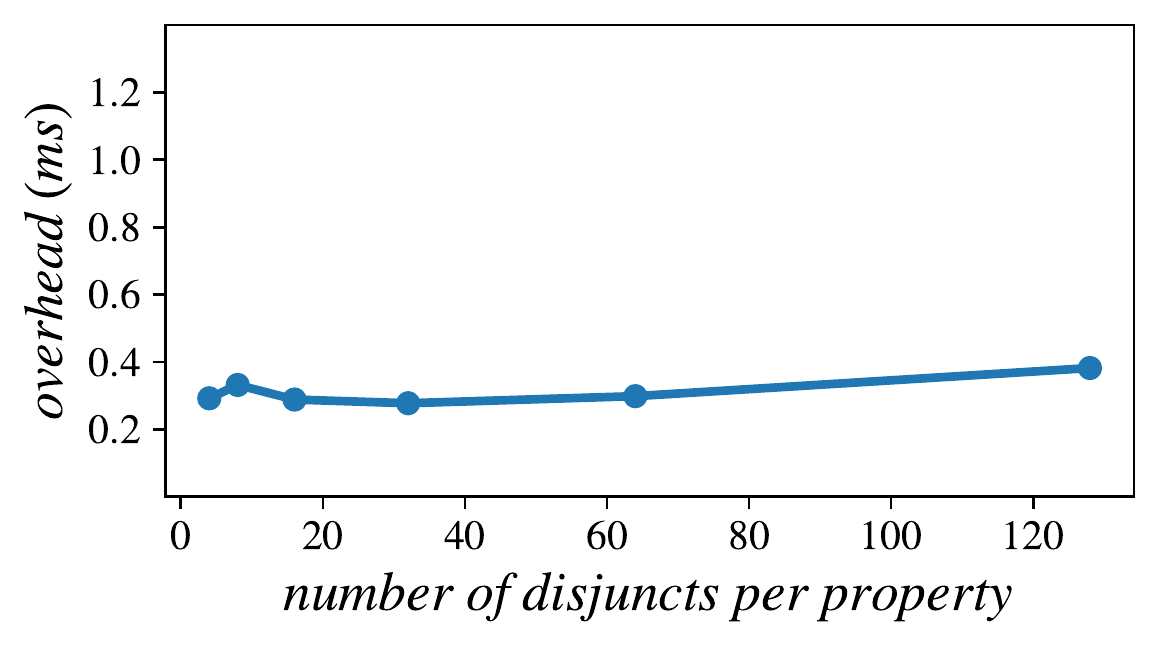}}
\end{subfigure}

\begin{subfigure}{0.5\columnwidth}
\resizebox{\textwidth}{!}{
\includegraphics{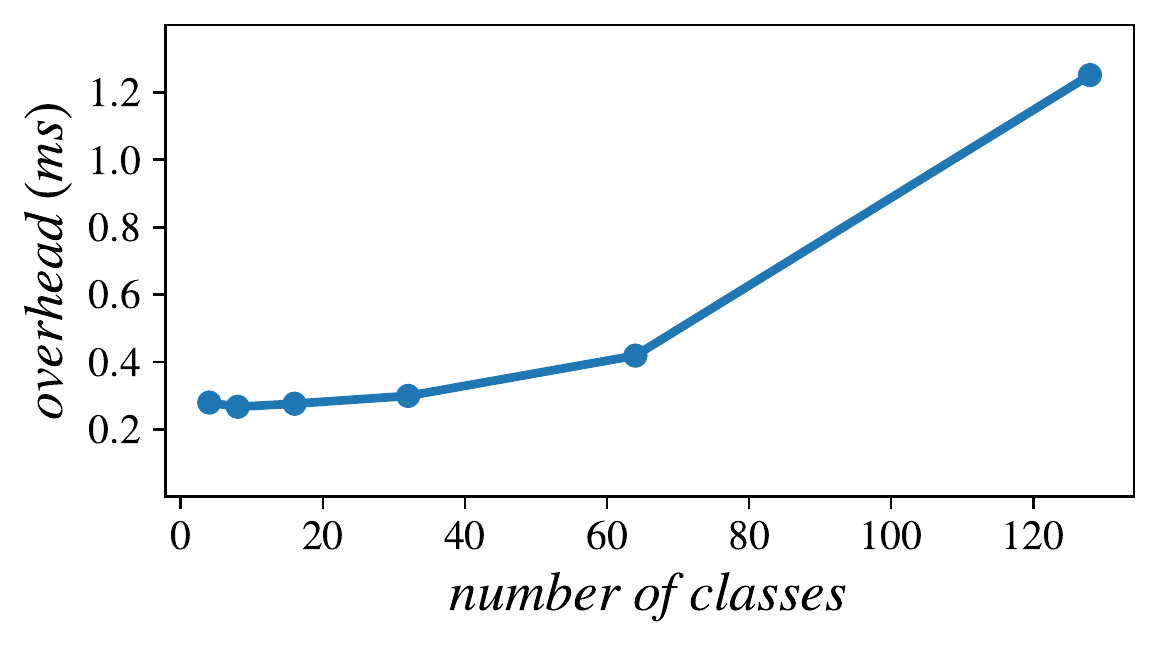}}
\end{subfigure}%
\begin{subfigure}{0.5\columnwidth}
\resizebox{\textwidth}{!}{
\includegraphics{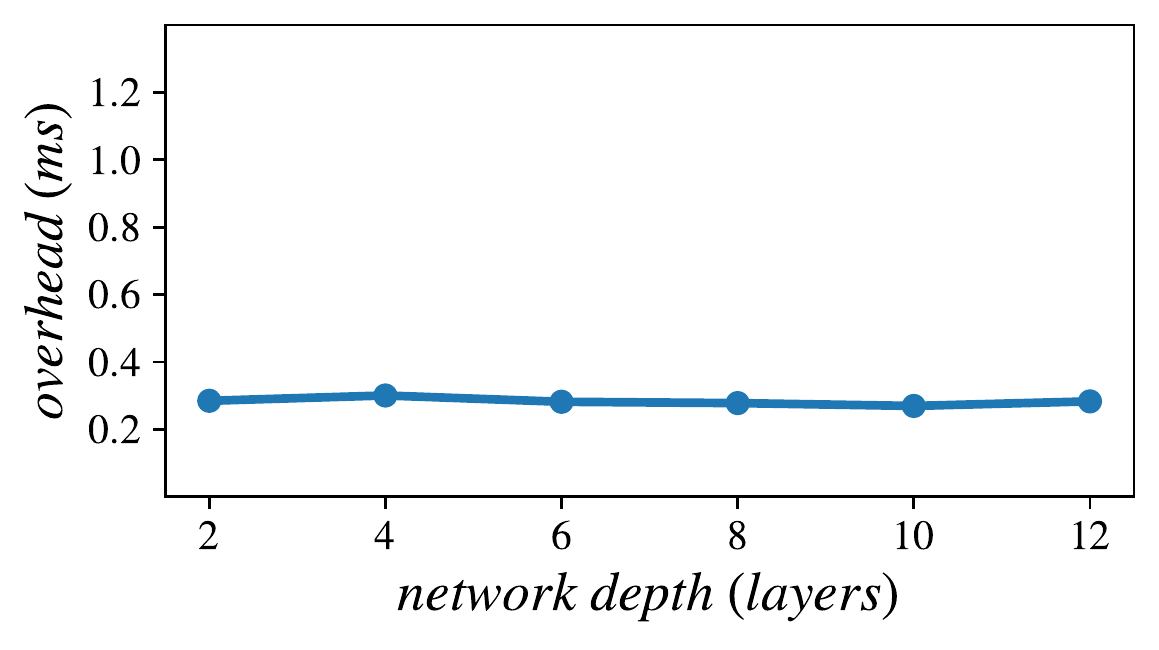}}
\end{subfigure}%
\caption{
	Absolute overhead in milliseconds introduced by the SC-Layer as either the number of properties, i.e., safe-ordering constraints ($\alpha$), the number of disjuncts per property ($\beta$), the number of classes ($m$), or the network depth ($\delta$) are varied.
	In each plot, the respective parameter varies according to the values on the x-axis, and all other parameters take a default value of $\alpha = 4$, $\beta = 4$, $m = 8$, and $\delta = 6$.
	As the depth of the network varies, the number of neurons in each layer remains fixed at 1,000 neurons.
	Reported overheads are averaged over 5 trials.
}
\label{fig:synth_overhead}
\end{figure}

Safe ordering constraints are capable of expressing a wide range of compelling safety specifications.
Moreover, our \Fselfrepair transformer is a powerful, general tool for ensuring safety with respect to arbitrarily complex safe-ordering constraints, comprised of many conjunctive and disjunctive clauses.
Notwithstanding, the properties presented in our evaluation thus far have been relatively simple.
In this section we explore more complex safe-ordering constraints, and describe experiments that lend insight as to which factors most impact the scalability of our approach.

To this end, we designed a family of synthetic datasets with associated safety constraints that are randomly generated according to several specified parameters, allowing us to assess how aspects such as the number of properties ($\alpha$), the number of disjunctions per property ($\beta$), and the dimension of the output vector ($m$) impact the run-time overhead.
In our experiments, we fix the input dimension, $n$, to be 10.
Each dataset is parameterized by $\alpha$, $\beta$, and $m$, and denoted by $\mathcal D(\alpha, \beta, m)$; the procedure for generating these datasets is provided in Appendix~\ref{app:synthetic_details}.

We use a dense network with six hidden layers of 1,000 neurons each as a baseline, trained on $\mathcal D(4,4,8)$.
Table~\ref{tab:overhead} shows the overhead introduced by applying \Fselfrepair to our baseline network.
We see that the average overhead is approximately a quarter of a millisecond per instance, accounting for a $10\times$ increase in prediction time.
Figure~\ref{fig:synth_overhead} provides a more complete picture of the overhead as we vary the number of safe-ordering constraints ($\alpha$), the number of disjuncts per constraint ($\beta$), the number of classes ($m$), or the depth of the network ($\delta$).

We observe that among these parameters, the overhead is sensitive only to the number of classes.
This is to be expected, as the complexity of the \Frepair scales directly with $m$ (see Section~\ref{sec:repair:complexity}),
requiring a topological sort of the $m$ elements of the network's output vector.
On the other hand, perhaps surprisingly, increasing the complexity of the safety constraints through either additional safe-ordering constraints or larger disjunctive clauses in the ordering constraints had little effect on the overhead.
While in the \emph{worst case} the complexity of the \Frepair is also dependent on these parameters (Section~\ref{sec:repair:complexity}), if \Fsolve finds a satisfiable disjunct quickly, it will short-circuit.
The average-case complexity of \Fsolve is therefore more nuanced, depending to a greater extent on the specifics of the constraints rather than simply their size.
Altogether, these observations suggest that the topological sort in \Frepair tends to account for the majority of the overhead.

Finally, the results in Figure~\ref{fig:synth_overhead}
concur with what we observed in Section~\ref{sec:eval:cifar}; namely that the overhead is independent of the size of the network.



\section{Related Work}
\label{sec:related}

\label{sec:related:nnverif}

\paragraph{Static Verification and Repair of Neural Networks} 
A number of approaches for verification of already-trained neural networks have been presented in recent years. 
They have focused on verifying safety
properties similar to our safe-ordering constraints.
Abstract interpretation approaches \cite{gehr18,singh19} verify properties that
associate polyhedra with pre- and postconditions. 
Reluplex \cite{katz17} encodes a network's semantics as a system of constraints, and poses verification as constraint satisfiability.
These approaches can encode safe-ordering constraints, which are a special case of polyhedral postconditions,
but they do not provide an effective means to construct safe networks.
Other verification approaches~\cite{huang2017safety,urban20} do not address safe ordering.

Many of the above approaches can provide counterexamples when the network is unsafe, but none of them are capable of repairing the network.  
A recent repair approach~\cite{sotoudeh2021provable}
can provably repair neural networks that have piecewise-linear activations with respect to safety specifications expressed using polyhedral pre- and postconditions. 
In contrast to our transparency guarantee,
they rely on heuristics to favor accuracy preservation.

\paragraph{Safe-by-Construction Learning}
Recent efforts seek to learn
neural networks that are correct by construction. 
Some approaches~\cite{fischer2019dl2,li19logic,madry2018towards} modify the learning objective by adding a penalty for unsafe or incorrect behavior, but they do not provide a safety guarantee for the learned network. 
Balancing accuracy against the modified learning objective is also a concern.
In our work we focus on techniques that provide guarantees without requiring external verifiers.

As discussed in Section~\ref{sec:eval:acas},
ART~\cite{lin20} aims to learn networks that satisfy
safety specifications by updating the loss function used in training.
Learning is not guaranteed to converge to zero correctness loss, and the resulting model may not be as accurate as one trained with conventional methods.
In contrast, our program transformer is guaranteed to produce a safe network that preserves accuracy.

A similar approach is presented in \cite{mirman2018differentiable}
to enforce local robustness for all input samples in the training dataset.
This technique also updates the learning objective and uses a differentiable abstract interpreter for over-approximating the set of reachable outputs.
For both this approach and that of \cite{lin20}, the run time of the differentiable abstract interpreter depends heavily on the size and complexity of the network, and it may be difficult or expensive to scale them to realistic architectures.

An alternative way to achieve correct-by-construction learning is to modify the architecture of the neural model.
This approach has been employed to construct networks that have a fixed Lipschitz constant~\cite{anil19,li2019preventing,trockman2021orthogonalizing}, a relational property that is useful for certifying local robustness and ensuring good training behavior.
Recent work~\cite{leino21,leino2021relaxing} shows how to construct models that achieve relaxed notions of global robustness, where the network is allowed to selectively abstain from prediction at inputs where local robustness cannot be certified. 
\cite{donti2021enforcing} use optimization layers to enforce stability properties of neural network controllers.
These techniques are closest to ours in spirit, although we focus on safety specifications, and more specifically safe-ordering constraints, which have not been addressed previously in the literature.

\label{sec:related:dynrepair}
\paragraph{Shielding Control Systems}
 
Recent approaches have proposed ensuring safety of control systems by
constructing run-time check-and-correct mechanisms, also referred to as {\em
shields} ~\cite{bloem2015shield,alshiekh2018safe,zhu19inductive}. Shields check
at run time if the system is headed towards an unsafe state and provide
corrections for potentially unsafe actions when necessary. To conduct these run-time checks, shields need access to a model of the
environment that describes the environment dynamics, i.e., the effect of controller
actions on environment states. Though shields and our proposed \Frepair share
the run-time check-and-correct philosophy, they are designed for different
problem settings. 

\paragraph{Recovering from Program Errors}
Embedding run-time checks into a program to ensure safety is a familiar
technique in the program verification literature.  Contract checking \cite{meyer92,findler02}, run-time verification \cite{havelund01}, and dynamic type checking
are all instances of such run-time checks. If a run-time check fails, the program terminates before violating the property. 
A large body of work also exists on 
gracefully recovering from errors caused by software issues
such as divide-by-zero, null-dereference, memory corruption, and
divergent loops \cite{long14,rinard04,kling12,berger06,qin05,perkins09}. 
These approaches are particularly relevant in the context of long-running
programs, when aiming to repair state just enough
so that computation can continue.


\section{Conclusion and Future Directions}
\label{sec:conclusion}

We presented a method for transforming a neural network into a \linebreak safe-by-construction \emph{self-correcting network}, termed \arnet, without harming the accuracy of the original network. This serves as a practical tool for providing safety with respect to a broad class of safety specifications, namely, \emph{safe-ordering constraints}, that we characterize in this work.

Unlike prior approaches, 
our technique guarantees safety without further training or modifications to the network's parameters.
Furthermore, the scalability of our approach is not limited by the size or architecture of the model being repaired. This allows it to be applied to large, state-of-the-art models, which is impractical for most other existing approaches.

A potential downside to our approach is the run-time overhead introduced by the SC-Layer.
We demonstrate in our evaluation that our approach maintains small overheads (less than one millisecond per instance), due to our vectorized implementation, which leverages GPUs for large-scale parallelism.

In future work, we plan to leverage the differentiability of the SC-Layer to further explore training against the repairs made by the SC-Layer, as this can potentially lead to both accuracy and safety improvements.

\bibliographystyle{splncs04}
\bibliography{bibfile}

\clearpage
\appendix


\section{Proofs}
\label{app:proofs}

\noindent
\textbf{Theorem~\ref{thm:preservation}} ~(Accuracy Preservation).
\textit{
Given a neural network, $f : \mathbb{R}^n \to \mathbb{R}^m$, and set of constraints, $\properties$,
let $f^\Phi := SC^\Phi(f)$
and let $F^O: \mathbb{R}^n \to [m]$ be the oracle classifier.
Assume that $SC$ satisfies transparency.
Further, assume that accuracy is consistent with safety, i.e.,
$$
\forall x\in\mathbb{R}^n~.~ \exists y~.~\Phi(x, y) ~\wedge~ \argmax_i\{y_i\} = F^O(x).
$$
Then,
$$
\forall x \in \mathbb{R}^n~.~ F(x) = F^O(x) \implies F^\Phi(x) = F^O(x)
$$
}
\begin{proof}
Let $x \in \mathbb{R}^n$ such that $F(x) = F^O(x)$.
By hypothesis, we have that $\exists y~.~\Phi(x, y) \wedge \argmax_i\{y_i\} = F^O(x)$,
hence we can apply Property~\ref{thm:transparency} to conclude that $F^\Phi(x) = F(x) = F^O(x)$.
\end{proof}

\paragraph{\ls`SC` is a Self-Correcting Transformer}
We now prove that the transformer presented in Algorithm~\ref{alg:self_repair}, \Fselfrepair,
is indeed self-correcting; i.e., it satisfies Properties~\ref{def:sr-transformer}\ref{thm:safety}
and~\ref{def:sr-transformer}\ref{thm:forewarning}.
Recall that this means that $\arnet$ will either return safe outputs vectors, or in the event that $\Phi$ is inconsistent at a point, and \emph{only} in that event, return $\bot$.

Let $x : \mathbb{R}^n$ be an arbitrary vector.
If $\Phi(x, f(x))$ is initially satisfied, the \arlayer does not modify
the original output $y = f(x)$, and Properties~\ref{def:sr-transformer}\ref{thm:safety}
and \ref{def:sr-transformer}\ref{thm:forewarning} are trivially satisfied.
If $\Phi(x, f(x))$ does not hold, we will rely on two key properties of \Fsolve and \Freorder to establish that \Fselfrepair is self-correcting.
The first, Property~\ref{thm:solvespec}, requires that \Fsolve either return $\bot$, or else return ordering constraints that are sufficient to establish $\Phi$.

\begin{Property}[\Fsolve]
\label{thm:solvespec}
Let $\Phi$ be a set of safe-ordering constraints, $x : \mathbb{R}^n$ and $y : \mathbb{R}^m$ two
vectors.\\
Then $q = \Fsolve(\Phi, x, y)$ satisfies the following properties:

\begin{enumerate}[label=(\roman*),font=\itshape]
\item \label{thm:solvebot} $q = \bot \Longleftrightarrow \forall y'~.~\lnot \Phi(x, y')$
\item \label{thm:solverep} $q \neq \bot \implies (~\forall y'~.~q(y') \implies \Phi(x, y')~)$
\end{enumerate}
\end{Property}
\begin{proof}
The first observation is that the list of ordering constraints in
$Q_p := \Fprioritize(Q_x, y)$ accurately models
the initial set of safety constraints $\Phi$, i.e.,
\begin{equation}
\label{eqn:prspec}
\forall y'~.~\Phi(x, y') \Longleftrightarrow (~\exists q \in Q_p~.~q(y')~)
\end{equation}
This stems from the definition of the disjunctive normal form,
and from the fact that \Fprioritize only performs a permutation of the disjuncts.

We also rely on the following loop invariant, stating that all disjuncts considered so far, when iterating over
$\Fprioritize(Q_x, y)$, were unsatisfiable:
\begin{equation}
\label{eqn:solve_loopinv}
\forall q \in Q_p~.~\idx(q, Q_p) < \idx(q_i, Q_p)
\implies (~\forall y~.~\neg q(y)~)
\end{equation}

Here, $\idx(q,Q_p)$ returns the index of constraint $q$ in the list $Q_p$.
This invariant is trivially true when entering the loop,
since the current $q_i$ is the first element of the list.
Its preservation relies on $\Fissat(q)$ correctly determining whether $q$ is satisfiable, i.e., $\Fissat(q) \Longleftrightarrow \exists y~.~q(y)$~\cite{graphissat}.

Combining these two facts, we can now establish that \linebreak \Fsolve satisfies \ref{thm:solvespec}\ref{thm:solvebot} and \ref{thm:solvespec}\ref{thm:solverep}.
By definition, \linebreak $\Fsolve(\Phi, x, y)$ outputs $\bot$ if and only if it traverses
the entire list $Q_p$, never returning a $q_i$.
From loop invariant~\ref{eqn:solve_loopinv}, this is equivalent to
$\forall q \in Q_p.~\forall y'.~\lnot q(y')$, which finally yields
property~\ref{thm:solvespec}\ref{thm:solvebot} from
equation~\ref{eqn:prspec}.
Conversely, if \linebreak $\Fsolve(\Phi, x, y)$ outputs $q \neq \bot$, then $q \in Q_p$.
We directly obtain property~\ref{thm:solvespec}\ref{thm:solverep} as,
for any $y' : \mathbb{R}^m$, $q(y')$ implies that $\Phi(x, y')$ by application
of equation~\ref{eqn:prspec}
\end{proof}

Next, Property~\ref{thm:reorderspec} states that $\Freorder$ correctly permutes the output of the network to satisfy the constraint that it is given.
Combined with Property~\ref{thm:solvespec}, this is sufficient to show that \Fselfrepair is a self-correcting transformer (Theorem~\ref{thm:sr-sound}).

\begin{Property}[\Freorder]
\label{thm:reorderspec}
Let $q$ be a satisfiable ordering constraint, and $y : \mathbb{R}^m$ a vector.
Then $\Freorder(q, y)$ satisfies $q$.
\end{Property}
\begin{proof}
Let $y_i < y_j$ be an atom in $q$. Reusing notation from Algorithm~\ref{alg:repair:reorder}, let $y' = \Freorder(q, y)$, $y^s := \Fdescsort(y)$, and $\pi := \Ftopsort(\linebreak\Fordergraph(q), y)$.
We have that $(j, i)$ is an edge in $\Fordergraph(q)$, which implies that $\pi(j) < \pi(i)$ by Equation~\ref{eqn:topsort1}.
Because the elements of $y$ are sorted in descending order, and assumed to be distinct (Definition~\ref{def:safe-ordering}), we obtain that $y^s_{\pi(i)} < y^s_{\pi(j)}$, i.e., that $y'_i < y'_j$. 
\end{proof}

\noindent
\textbf{Theorem~\ref{thm:sr-sound}} ~(\Fselfrepair is a self-correcting transformer).
\textit{
\Fselfrepair (Algorithm~\ref{alg:self_repair}) satisfies conditions \emph{(i)} and \emph{(ii)} of Definition~\ref{def:sr-transformer}.
}
\begin{proof}
By definition of Algorithm~\ref{alg:self_repair},
$\Fsolve(\Phi, x, y) = \bot$ if and only if $\arnet(x) = \Fselfrepair(\Phi)(f)(x)$ outputs $\bot$.
We derive from Property~\ref{thm:solvespec}\ref{thm:solvebot} that this is equivalent to
$\forall y'.~\lnot \Phi(x, y')$, which corresponds exactly to Property~\ref{def:sr-transformer}\ref{thm:forewarning}.
Conversely, if $\Phi$ is satisfiable for input $x$, i.e., $\exists y'.~\Phi(x, y')$,
then \linebreak $\Fsolve(\Phi, x, y)$ outputs $q \neq \bot$.
By definition, we have $\arnet(x) = \Freorder(q, y)$, which satisfies $q$
by application of Property~\ref{thm:reorderspec}, which in turn implies
that $\Phi(x, \arnet(x))$ by application of
Property~\ref{thm:solvespec}\ref{thm:solverep}.
\end{proof}

\paragraph{\ls`SC` is Transparent}
Now that we have demonstrated that our approach produces safe-by-construction networks,
we next prove that it also preserves the top predicted class when possible, i.e.,
that \Fselfrepair satisfies \emph{transparency}, as formalized in Property~\ref{thm:transparency}.

Let $x : \mathbb{R}^n$ be an arbitrary vector.
As in the previous section, if $\Phi(x, f(x))$ is initially satisfied, transparency trivially
holds, as the correction layer does not modify the original output $f(x)$.
When $\Phi(x, f(x))$ does not hold, we will rely on several additional properties about \Fsolve, \Freorder, and \Fordergraph.
The first, Property~\ref{thm:tograph}, states that whenever the index of the network's top prediction is a root of the graph encoding of $q$ used by \linebreak \Fsolve and \Freorder, then there exists an output which satisfies $q$ that preserves that top prediction.

\begin{Property}[\Fordergraph]
\label{thm:tograph}
Let $q$ be a satisfiable, disjunction-free ordering constraint,
and $y : \mathbb{R}^m$ a vector. Then,
\begin{align*}
\argmax_i\{y_i\} \in \rootg(\Fordergraph(q)) ~\Longleftrightarrow~\\
  \exists y'.~ q(y') ~\wedge~ \argmax_i\{y_i\} = \argmax_i\{y_i'\}
\end{align*}
\end{Property}

The intuition behind this property is that $i^* := \argmax_i\{y_i\}$ belongs to the roots
of $\Fordergraph(q)$ if and only if there is no $y_{i^*} < y_j$ constraint in $q$; hence
since $q$ is satisfiable, we can always permute indices in a solution $y'$ to have
$\argmax_i\{y'_i\} = i^*$.
Formally, Lemma~\ref{lem:topsort-invariant} in Section~\ref{sec:impl:stable_topological_sort} entails this property, as it shows that the permutation returned by \Ftopsort satisfies it.

Next, Property~\ref{thm:refinesolve} formalizes the requirement that whenever \linebreak \Fsolve returns a constraint (rather than $\bot$), then that constraint will not eliminate any top-prediction-preserving solutions that would otherwise have been compatible with the full set of safe-ordering constraints $\Phi.$

\begin{Property}[\Fsolve]
\label{thm:refinesolve}
Let $\Phi$ be a set of safe-ordering constraints, $x : \mathbb{R}^n$ and $y : \mathbb{R}^m$ two vectors, and  $q = \Fsolve(\Phi, x, y)$. Then,
\begin{align*}
&q \neq \bot ~\wedge~ \left(\exists y'.~ \Phi(x, y') ~\wedge~ \argmax_i\{y_i\} = \argmax_i\{y_i'\}\right) \implies \\
&\exists y'.~q(y') ~\wedge~ \argmax_i\{y_i\} = \argmax_i\{y_i'\}
\end{align*}
\end{Property}
\begin{proof}
Let us assume that $q \neq \bot$, and that $\exists y'.~\Phi(x, y') \wedge \argmax_i\{y_i\} = \argmax_i\{y'_i\}$.
We will proceed by contradiction, assuming that there does not exist $y''$ such that
$q(y'')$ and $\argmax_i\{y_i\} = \argmax_i\{y''_i\}$, which
entails that $\argmax_i\{y_i\} \not\in \rootg(\Fordergraph(q))$
by application of Property~\ref{thm:tograph}.
In combination with the specification of \Fprioritize (Property~\ref{prop:prioritize}),
this implies that any $q' \in Q_p$ such that
$\exists y'.~q'(y') \wedge \argmax_i\{y_i\} = \argmax_i\{y'_i\}$ occurs before $q$ in
$\Fprioritize(Q_x, y)$, i.e., $\idx(q', Q_p) \allowbreak < \idx(q, Q_p)$.
From loop invariant~\ref{eqn:solve_loopinv}, we therefore conclude that there does
not exist such a $q' \in Q_p$, which contradicts the hypothesis $\Phi(x, y')$
by application of Equation~\ref{eqn:prspec}.
\end{proof}

Lastly, Property~\ref{thm:refinereorder} states that \Freorder (Algorithm~\ref{alg:repair:reorder}) will always find an output that preserves the original top prediction, whenever the constraint returned by \Fsolve allows it.
This is the final piece needed to prove Theorem~\ref{thm:sr-transparency}, the desired result about the self-correcting transformer.

\begin{Property}[\Freorder]
\label{thm:refinereorder}
Let $q$ be a satisfiable term, and $y : \mathbb{R}^m$ a vector.
Then,
\begin{align*}
&(~\exists y'.~q(y') \wedge \argmax_i\{y_i\} = \argmax_i\{y_i'\}~)\\
&\implies \argmax_i\{\Freorder(q, y)_i\} = \argmax_i\{y_i\}
\end{align*}
\end{Property}
\begin{proof}
Assume that there exists $y'$ such that $q(y')$ and
$\argmax_i\{y_i\} = \argmax_i\{y'_i\}$.
This entails that $\argmax_i(y_i) \in \rootg(\Fordergraph(q))$ (Property~\ref{thm:tograph}),
which in turn implies that $\pi(\argmax_i\{y_i\})$ is 0
(property~\ref{eqn:topsort2}).
By definition of a descending sort, we have that
$\argmax_i\{\Freorder(q, y)_i\} = j$, such that $\pi(j) = 0$,
hence concluding that $j = \argmax_i\{y_i\}$ by injectivity of $\pi$.
\end{proof}

\noindent
\textbf{Theorem~\ref{thm:sr-transparency}} ~(Transparency of \Fselfrepair).
\textit{
\Fselfrepair, the self-correcting transformer described in Algorithm~\ref{alg:self_repair} satisfies Property~\ref{thm:transparency}.
}
\begin{proof}
That the \Fselfrepair transformer satisfies transparency is straightforward given Properties~\ref{thm:tograph}-\ref{thm:refinereorder}. 
Let us assume that there exists $y'$ such that
$\Phi(x, y')$ and \linebreak $\argmax_i\{y'_i\} = F(x)$.
By application of Property~\ref{thm:solvespec}\ref{thm:solvebot},
this implies that \linebreak
$\Fsolve(\Phi, x, f(x))$ outputs $q \neq \bot$,
and therefore that there exists $y'$ such that $q(y')$ and $\argmax\{y'_i\} = F(x)$
by application of Property~\ref{thm:refinesolve}, since $F(x)$ is defined as $\argmax_i\{f_i(x)\}$.
Composing this fact with Property~\ref{thm:refinereorder}, we obtain that $F^\Phi(x) = F(x)$, since
$F^\Phi(x) = \argmax_i\{\arnet_i(x)\}$ by definition.
\end{proof}


\section{Vectorizing Self-Correction}
\label{app:alg_details}

Several of the subroutines of \Fsolve and \Freorder (Algorithms \ref{alg:repair:solve} and \ref{alg:repair:reorder} presented in Section~\ref{sec:repair}) operate on an \Fordergraph, which represents a conjunction of ordering literals, $q$.
An \Fordergraph contains a vertex set, $V$, and edge set, $E$, where $V$ contains a vertex, $i$, for each class in $\{0, \ldots, m-1\}$, and $E$ contains an edge, $(i, j)$, from vertex $i$ to vertex $j$ if the literal $y_j < y_i$ is in $q$.
We represent an \Fordergraph as an $m\times m$ adjacency matrix, $M$, defined according to Equation~\ref{eq:adjacency_matrix}. 
\begin{equation}
\label{eq:adjacency_matrix}
M_{ij} :=
\begin{cases}
1 &\text{if $(i, j) \in E$;~ i.e., $y_j < y_i \in q$} \\
0 &\text{otherwise}
\end{cases}
\end{equation}
Section~\ref{sec:stable_topo_sort} describes the matrix-based algorithm that we use to conduct the stable topological sort that \Freorder (Algorithm~\ref{alg:repair:reorder}) depends on.
It is based on a classic parallel algorithm due to \cite{dekel81graph}, which we modify to ensure that \Fselfrepair satisfies transparency (Property~\ref{thm:transparency}).
Section~\ref{sec:cycle-detection} describes our approach to cycle detection, which is able to share much of its work with the topological sort.
Finally, Section~\ref{sec:impl:prioritize} discusses efficiently prioritizing ordering constraints, needed to ensure that \Fselfrepair satisfies transparency.

\subsection{Stable Topological Sort}
\label{sec:stable_topo_sort}
\label{sec:impl:stable_topological_sort}

\begin{algorithm}[t]
\small
\vspace{0.5em}
\KwIn{A graph, $G$, represented as an $m\times m$ adjacency matrix, and a vector, $y : \mathbb{R}^m$}
\KwResult{A permutation, $\pi : [m] \to [m]$}
\vspace{0.5em}
\Fn{\Ftopsort{$G ~~,~~ y$}}{\vspace{0.25em}
	$P ~:=~ \texttt{all\_pairs\_longest\_paths}(G)$\;\vspace{0.25em}
	$\forall~ i,j \in [m] ~~.~~ P'_{ij} ~:=~ \begin{cases}
		y_i &\text{if $P_{ij} \geq 0$}\\ 
		\infty &\text{otherwise}
	\end{cases}$\;\vspace{0.5em}
	\label{line:sort:min_ancestor}
	$\forall~ j \in [m] ~~.~~ v_j ~:=~ \min\displaylimits_i\left\{~P'_{ij}~\right\}$
		\atcp{set the value of each vertex to the}\vspace{-0.25em}
		\atcp{smallest value among its ancestors}\vspace{0.25em}
	$\forall~ j \in [m] ~~.~~ d_j ~:=~ \max\displaylimits_i\left\{~P_{ij}~\right\}$
		\atcp{calculate the depth of each vertex}\vspace{0.25em}
	\textbf{return} \texttt{argsort($[~\forall j \in [m] ~~.~~ (-v_j,\, d_j)~]$)}
		\atcp{break ties in favor of minimum depth}
}
\caption{Stable Topological Sort}
\label{alg:top_sort}
\end{algorithm}

Our approach builds on a parallel topological sort algorithm given by \cite{dekel81graph}, which is based on constructing an \emph{all pairs longest paths} (APLP) matrix.
However, this  algorithm is not \emph{stable} in the sense that the resulting order depends only on the graph, and not on the original order of the sequence, even when multiple orderings are possible.
While for our purposes this is sufficient for ensuring safety, it is not for transparency.
We begin with background on constructing the APLP matrix, showing that it is compatible with a vectorized implementation, and then describe how it is used to perform a stable topological sort.

\paragraph{All Pairs Longest Paths}
\label{sec:impl:aplp}
The primary foundation underpinning many of the graph algorithms in this section is the \emph{all pairs longest paths} (APLP) matrix, which we will denote by $P$.
On acyclic graphs, $P_{ij}$ for $i,j \in [m]$ is defined to be the length of the \emph{longest} path from vertex $i$ to vertex $j$.
Absent the presence of cycles, the distance from a vertex to itself, $P_{ii}$, is defined to be 0.
For vertices $i$ and $j$ for which there is no path from $i$ to $j$, we let $P_{ij} = -\infty$.

We compute $P$ from $M$ using a matrix-based algorithm from \cite{dekel81graph}, which requires taking $O(\log{m})$ matrix \emph{max-distance products}, where the max-distance product is equivalent to a matrix multiplication where element-wise multiplications have been replaced by additions and element-wise additions have been replaced by the pairwise maximum.
That is, a matrix product can be abstractly written with respect to operations $\otimes$ and $\oplus$ according to Equation~\ref{eq:matrix_mult}, and the max-distance product corresponds to the case where $x \otimes y := x + y$ and $x \oplus y := \max\{x, y\}$.
\begin{equation}
\label{eq:matrix_mult}
(AB)_{ij} := (A_{i1}\otimes B_{1j}) \oplus \ldots \oplus (A_{ik}\otimes B_{kj})
\end{equation}

Using this matrix product, $P = P^{2^{\lceil\log_2(m)\rceil}}$ can be computed recursively from $M$
by performing a fast matrix exponentiation, as described in Equation~\ref{eq:aplp}.
\begin{align}
P^k &= P^{\nicefrac{k}{2}}P^{\nicefrac{k}{2}} & 
\label{eq:aplp}
P^1_{ij} &= \begin{cases}
	1 &\text{if $M_{ij} = 1$} \\
	0 &\text{if $M_{ij} = 0 ~\land~ i = j$} \\
	-\infty &\text{otherwise}
\end{cases}
\end{align}

\begin{figure*}[t]
\centering
\noindent
\begin{subfigure}{0.4\textwidth}
\centering
\begin{tikzpicture}[
vertex/.style={
    circle, 
    draw=black, 
    fill=white, 
    thin, 
    align=center,
    text width=5mm,
    execute at begin node=\scriptsize\setlength{\baselineskip}{3ex}},
function/.style={
    rectangle, 
    draw=black!80, 
    fill=white, 
    thin, 
    minimum size=0.5in, 
    inner sep=0.25in},
groupnode/.style={
    rectangle, 
    draw=black, 
    fill=black!5, 
    thin, 
    minimum size=2cm,
    inner sep=0.125in},
]

\node[vertex] (y0) at (0,0) {0\\(2)};
\node[vertex] (y1) at (1.618,-1.176) {1\\(3)};
\node[vertex] (y2) at (1.0,-3.08) {2\\(1)};
\node[vertex] (y3) at (-1.0,-3.08) {3\\(4)};
\node[vertex] (y4) at (-1.618,-1.176) {4\\(5)};

\draw[thick, ->] (y0) -- (y4);
\draw[thick, ->] (y1) -- (y2);
\draw[thick, ->] (y1) -- (y3);
\draw[thick, ->] (y1) -- (y4);
\draw[thick, ->] (y3) -- (y2);
\end{tikzpicture}
\caption{}
\label{fig:sort:init}
\end{subfigure}%
\begin{subfigure}{0.4\textwidth}
\centering
\begin{tikzpicture}[
vertex/.style={
    circle, 
    draw=black, 
    fill=white, 
    thin, 
    align=center,
    text width=5mm,
    execute at begin node=\scriptsize\setlength{\baselineskip}{3ex}},
function/.style={
    rectangle, 
    draw=black!80, 
    fill=white, 
    thin, 
    minimum size=0.5in, 
    inner sep=0.25in},
groupnode/.style={
    rectangle, 
    draw=black, 
    fill=black!5, 
    thin, 
    minimum size=2cm,
    inner sep=0.125in},
]

\node[vertex] (y0) at (0,0) {\scriptsize\baselineskip=50pt 0 (2,0)};
\node[vertex] (y1) at (1.618,-1.176) {1 (3,0)};
\node[vertex] (y2) at (1.0,-3.08) {2 (1,2)};
\node[vertex] (y3) at (-1.0,-3.08) {3 ({\color{red}3},1)};
\node[vertex] (y4) at (-1.618,-1.176) {4 ({\color{red}2},1)};

\draw[thick, ->] (y0) -- (y4);
\draw[thick, ->] (y1) -- (y2);
\draw[thick, ->] (y1) -- (y3);
\draw[thick, ->] (y1) -- (y4);
\draw[thick, ->] (y3) -- (y2);
\end{tikzpicture}
\caption{}
\label{fig:sort:step2}
\end{subfigure}
\caption{
	Example trace of Algorithm~\ref{alg:top_sort}.
	\textbf{(\subref{fig:sort:init})}: The dependency graph and original logit values, $y$. 
	The values of each logit are provided; the non-bracketed number indicates the logit index
	and the number in brackets is the logit value, e.g., $y_0 = 2$.
	Arrows indicate a directed edge in the dependency graph; e.g., we require $y_4 < y_0$.
	\textbf{(\subref{fig:sort:step2})}: updated values passed into \texttt{argsort} as a tuple.
	For example, $y_4$ is assigned $(2,1)$, as its smallest ancestor ($y_0$) has logit value 2 in (\subref{fig:sort:init}) and its depth is 1;
	and $y_2$ is assigned value $(1,2)$ because its logit value in (\subref{fig:sort:init}), 1, is already smaller than that any of its parents, and its depth is 2.
	The values are sorted by \emph{decreasing} value and \emph{increasing} depth, thus the final order is $\langle y_1, y_3, y_0, y_4, y_2 \rangle$, corresponding to the permutation $\pi$, where $\pi(0)=2$, $\pi(1)=0$, $\pi(2)=4$, $\pi(3)=1$, and $\pi(4)=3$.
}
\label{fig:example_graph}
\end{figure*}
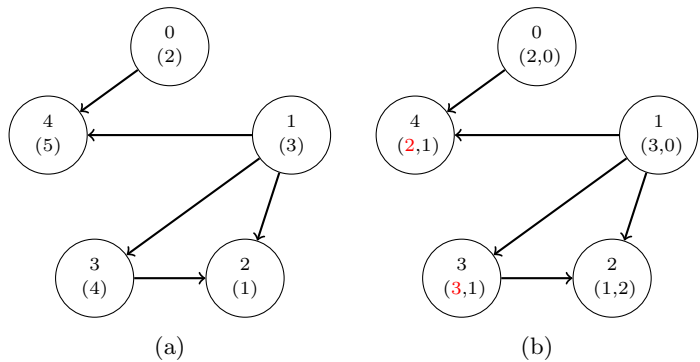

\paragraph{Stable Sort}

We propose a stable variant of the \cite{dekel81graph} topological sort, shown in Algorithm~\ref{alg:top_sort}.
Crucially, this variant satisfies Property~\ref{eqn:topsort2} (Lemma~\ref{lem:topsort-invariant}), which Section~\ref{sec:repair:reorder} identifies as sufficient for ensuring transparency.
Essentially, the value of each logit $y_j$ is adjusted so that it is at least as small as the smallest logit value corresponding to vertices that are parents of vertex $j$, including $j$ itself.
A vertex, $i$, is a parent of vertex $j$ if $P_{ij} \ge 0$, meaning that there is some path from vertex $i$ to vertex $j$ or $i = j$.
The logits are then sorted in descending order, with ties being broken in favor of minimum depth in the dependency graph.
The depth of vertex $j$ is the maximum of the $j^\text{th}$ column of $P_{ij}$, i.e., the length of the longest path from any vertex to $j$.
An example trace of Algorithm~\ref{alg:top_sort} is given in Figure~\ref{fig:example_graph}.
By adjusting $y_j$ into $v_j$ such that for all ancestors, $i$, of $j$,
$v_i \geq v_j$, we ensure each child vertex appears after each of its parents in the returned ordering--once ties have been broken by depth---as the child's depth will always be strictly larger than that of any of its parents since a path of length $d$ to an immediate parent of vertex $j$ implies the existence of a path of length $d+1$ to vertex $j$.

\begin{lemma}
\label{lem:topsort-invariant}
\Ftopsort satisfies Property~\ref{eqn:topsort2}.
\end{lemma}

\begin{proof}
Note that the adjusted logit values, $v$, are chosen according to Equation~\ref{eq:adjusted_logits}.
\begin{equation}
\label{eq:adjusted_logits}
v_j := \min_{i~|~ \text{$i$ is an ancestor of $j$} ~\lor~ i=j}\Big\{~y_i~\Big\}
\end{equation}

\noindent
We observe that \emph{(i)} for all root vertices, $i$, $v_i = y_i$, and \emph{(ii)} the root vertex with the highest original logit value will appear first in the topological ordering.
The former follows from the fact that the root vertices have no ancestors.
The latter subsequently follows from the fact that the first element in a valid topological ordering must correspond to a root vertex.
Thus if $\argmax_i\{y_i\} = i^* \in \rootg(g)$, then $i^*$ is the vertex with the highest logit value, and so by \emph{(ii)}, it will appear first in the topological ordering produced by \Ftopsort 
, establishing Property~\ref{eqn:topsort2}.
\end{proof}

\subsection{Cycle Detection}
\label{sec:cycle-detection}
\Fissat, a subroutine of \Fsolve (Algorithm~\ref{alg:repair:solve}) checks to see if an ordering constraint, $q$, is satisfiable by looking for any cycles in the corresponding dependency graph, $\Fordergraph(q)$.
Here we observe that the existence of a cycle can easily be decided from examining $P$, by checking if $P_{ii} > 0$ for some $i\in[m]$; i.e., if there exists a non-zero-length path from any vertex to itself.
Since $P_{ii} \geq 0$, this is equivalent to $\trace(P) > 0$.
While strictly speaking, $P_{ij}$, as constructed by \cite{dekel81graph}, only reflects the longest path from $i$ to $j$ in \emph{acyclic} graphs, it can nonetheless be used to detect cycles in this way, as for any $k \leq m$, $P_{ij}$ is guaranteed to be at least $k$ if there exists a path of length $k$ from $i$ to $j$,
and any cycle will have length at most $m$.

\subsection{Prioritizing Root Vertices}
\label{sec:impl:prioritize}
As specified in Property~\ref{prop:prioritize}, in order to satisfy transparency, the search for a satisfiable ordering constraint performed by \Fsolve must prioritize constraints, $q$, in which the original predicted class, $F(x)$, is a root vertex in $q$'s corresponding dependency graph.
We observe that root vertices can be easily identified using the dependency matrix $M$.
The in-degree, $d^\textit{in}_j$, of vertex $j$ is simply the sum of the $j^\text{th}$ column of $M$, given by Equation~\ref{eq:in-degree}.
Meanwhile, the root vertices are precisely those vertices with no ancestors, that is, those vertices $j$ satisfying Equation~\ref{eq:in-degree}.
\begin{equation}
\label{eq:in-degree}
d^\textit{in}_j = \sum_{i\in[m]}{M_{ij}} = 0
\end{equation}
In the context of \Fsolve, the subroutine \Fprioritize lists ordering constraints $q$ for which $d^\textit{in}_{F(x)} = 0$ in $\Fordergraph(q)$ before any other ordering constraints.
To save memory, we do not explicitly list and sort all the disjuncts of $Q_x$ (the DNF form of the active postconditions for $x$); rather we iterate through them one at a time.
This can be done by, e.g., iterating through each disjunct twice, initially skipping any disjunct in which $F(x)$ is not a root vertex, and subsequently skipping those in which $F(x)$ is a root vertex.


\section{Generation of Synthetic Data}
\label{app:synthetic_details}

In Section~\ref{sec:eval:synthetic}, we utilize a family of synthetic datasets with associated safe-ordering constraints that are randomly generated according to several specified parameters, allowing us to assess how aspects such as the number of constraints ($\alpha$), the number of disjunctions per constraint ($\beta$), and the dimension of the output vector ($m$) impact the run-time overhead.
In our experiments, we fix the input dimension, $n$, to be 10.
The synthetic data, which we will denote by $\mathcal D(\alpha, \beta, m)$, are generated according to the following procedure.
\begin{enumerate}[label=(\roman*)]
\item
	First, we generate $\alpha$ random safe-ordering constraints.
	The preconditions take the form $b_\ell \leq x \leq b_u$, where $b_\ell$ is drawn uniformly at random from $[0.0,1.0 - \epsilon]$ and $b_u := b_\ell + \epsilon$.
	We choose $\epsilon = 0.4$ in our experiments; as a result, the probability that any two preconditions overlap is approximately 30\%.
	The ordering constraints are disjunctions of $\beta$ randomly-generated cycle-free ordering graphs of $m$ vertices, i.e., $\beta$ disjuncts.
	Specifically, in each graph, we include each edge, $(i, j)$, for $i \neq j$ with equal probability, and require further that at least one edge is included, and the expected number of edges is $\gamma$ (we use $\gamma = 3$ in all of our experiments).
	Graphs with cycles are resampled until a graph with no cycles is drawn.
\item \label{step:data_for_properties}
	Next, for each safe-ordering constraint, $\phi$, we sample $\nicefrac{N}{\alpha}$ random inputs, $x$, uniformly from the range specified by the precondition of $\phi$. In all of our experiments we let $N = \text{2,000}$.
	For each $x$, we select a random disjunct from the postcondition of $\phi$, and find the roots of the corresponding ordering graph.
	We select a label, $y^*$ for $x$ uniformly at random from this set of roots,
	i.e., we pick a random label for each point that is consistent with the property for that point.
\item
	Finally, we generate $N$ random points that do not satisfy any of the preconditions of the $\alpha$ safe-ordering constraints.
	We label these points via a classifier trained on the $N$ labeled points already generated in \ref{step:data_for_properties}.
	This results in a dataset of $2N$ labeled points, where 50\% of the points are captured by at least one safe-ordering constraint.
\end{enumerate}

\end{document}